\documentclass[11pt]{article}
\pdfoutput=1
\usepackage{ifthen}
\newboolean{usenatbib}
\setboolean{usenatbib}{false}
\ifthenelse{\boolean{usenatbib}}{
  \usepackage[numbers,sort]{natbib}
}{
  \usepackage[style=alphabetic,maxalphanames=4,minalphanames=3,maxcitenames=4,maxbibnames=99,giveninits=false,doi=false,url=false,natbib]{biblatex}
  \addbibresource{refs.bib}

  \DeclareDelimFormat[bib]{nametitledelim}{\addspace}
}
\usepackage{bm}
\usepackage{fullpage}
\usepackage[english]{babel}
\usepackage[utf8]{inputenc}
\usepackage[T1]{fontenc}
\usepackage{ae} \usepackage{aecompl}
\usepackage{enumitem}
\setlist{nosep}
\usepackage{microtype}
\usepackage{booktabs}
\usepackage{hhline}
\usepackage{makecell}
\usepackage{etoolbox}
\apptocmd{\sloppy}{\hbadness 10000\relax}{}{}

\usepackage{bm}
\usepackage{fullpage}
\usepackage[english]{babel}
\usepackage[T1]{fontenc}
\usepackage{ae} \usepackage{aecompl}
\usepackage{enumitem}
\setlist{nosep}
\usepackage{microtype}
\usepackage{booktabs}
\usepackage{hhline}
\usepackage{makecell}
\usepackage{etoolbox}
\apptocmd{\sloppy}{\hbadness 10000\relax}{}{}

\usepackage{amsmath,amsthm,amssymb,mathtools,thmtools}
\usepackage{graphicx}
\usepackage{color}
\usepackage[dvipsnames]{xcolor}
\usepackage[textsize=scriptsize,disable]{todonotes}
\usepackage[colorlinks=true, allcolors=blue]{hyperref}
\usepackage[algo2e,ruled]{algorithm2e}
\usepackage{xfrac}
\usepackage[normalem]{ulem}

\hypersetup{colorlinks,
            breaklinks=true,
            linkcolor=blue,
            citecolor=blue,
            urlcolor=magenta,
            linktocpage,
            plainpages=false,
            bookmarks=false}

\usepackage{notations_arxiv}
\usepackage{comment}
\usepackage{tikz,pgfplots,pgfplotstable}
\usepgfplotslibrary{units}
\usepackage{tabularx} 
\usepackage{subfig}
\usepackage{array}
\usepgfplotslibrary{groupplots}

\title{Sample-Optimal Locally Private Hypothesis Selection\\ and the Provable Benefits of Interactivity}

\author{
    Alireza F. Pour\thanks{University of Waterloo, \texttt{alireza.fathollahpour@uwaterloo.ca}. }
    \and
    Hassan Ashtiani\thanks{McMaster University, \texttt{zokaeiam@mcmaster.ca}. Hassan Ashtiani is also a faculty affiliate at Vector Institute and supported by an NSERC Discovery Grant.}
    \and 
    Shahab Asoodeh\thanks{McMaster University, \texttt{Asoodeh@mcmaster.ca}. Shahab Asoodeh is also a faculty affiliate at Vector Institute and supported by an NSERC Discovery Grant}
}

\date{}
\numberwithin{equation}{section}

\begin{document}
\maketitle

\begin{abstract}%

We study the problem of hypothesis selection under the constraint of local differential privacy. Given a class $\cF$ of $k$ distributions and a set of i.i.d.\ samples from an unknown distribution $h$, the goal of hypothesis selection is to pick a distribution $\hat{f}$ whose total variation distance to $h$ is comparable with the best distribution in $\cF$ (with high probability). We devise an $\eps$-locally-differentially-private ($\eps$-LDP) algorithm that uses $\Theta\left(\frac{k}{\alpha^2\min \{\eps^2,1\}}\right)$ samples to guarantee that $d_{TV}(h,\hat{f})\leq \alpha + 9 \min_{f\in \cF}d_{TV}(h,f)$ with high probability. This sample complexity is optimal for $\eps<1$, matching the lower bound of \citet{gopi2020locally}. All previously known algorithms for this problem required
$\Omega\left(\frac{k\log k}{\alpha^2\min \{ \eps^2 ,1\}} \right)$ samples to work.

Moreover, our result demonstrates the power of interaction for $\varepsilon$-LDP hypothesis selection. Namely, it breaks the known lower bound of $\Omega\left(\frac{k\log k}{\alpha^2\min \{ \eps^2 ,1\}} \right)$ for the sample complexity of non-interactive hypothesis selection. Our algorithm breaks this barrier using only $\Theta(\log \log k)$ rounds of interaction.

To prove our results, we define the notion of \emph{critical queries} for a Statistical Query Algorithm (SQA) which may be of independent interest. Informally, an SQA is said to use a small number of critical queries if its success relies on the accuracy of only a small number of queries it asks. We then design an LDP algorithm that uses a smaller number of critical queries.

\end{abstract}

\section{Introduction}
One of the basic problems in statistical learning is hypothesis selection where we are given a set of i.i.d.\ samples from an unknown distribution $h$ and a set $\cF$ of $k$ distributions. The goal is to select a distribution in $\cF$ whose total variation distance to $h$ is close to the minimum total variation distance between $h$ and any distribution in $\cF$. This problem has been studied extensively in the literature \citep{yatracos1985rates,mahalanabis2007density, bousquet2019optimal, bousquet2022statistically, aliakbarpour2024hypothesis} and it is known that the number of samples required to solve this problem has a tight logarithmic dependency on $k$ for a general class of distributions. 

In many practical statistical estimation scenarios, data points contain sensitive information such as medical and financial records. This urges the study of hypothesis selection under differential privacy (DP) \citep{dwork2006calibrating} ---the de-facto privacy standard in machine learning. Central and local models are two common settings of DP. In the central model, a learning algorithm is differentially private if its output, given full access to a database, does not significantly change with a small perturbation in the input database (e.g., by changing one entry in the database).
In the local differential privacy (LDP) model \citep{warner1965randomized,kasiviswanathan2011can, duchi2013local}, however, the algorithm cannot access the database directly. Instead, it receives a privatized (e.g., randomized) version of each data point through privacy-preserving mechanisms. In this paper, we work with the notion of $\eps$-LDP; see Section~\ref{sec:ldp} for a more formal definition. In fact, the local model of privacy ensures a \textit{user-level} privacy and it has been the preferred model for implementation in industry, e.g., by Google, Apple, and Microsoft \citep{erlingsson2014rappor,apple,ding2017collecting, patent}. 

It is important to distinguish between non-adaptive and adaptive locally private algorithms. A non-adaptive algorithm chooses a privacy mechanism for each data point independently. On the other hand, an adaptive algorithm runs in multiple rounds and chooses each mechanism at each round based on the outcomes of mechanisms in the previous rounds \citep{duchi2018minimax, duchi2019lower, joseph2019role}.
This \emph{(sequential) interactivity} allows the algorithm to be more flexible in hiding private information and has led to a smaller sample complexity in several problems~\citep{han2018geometric, Exp_Seperation, acharya2020unified, acharya2022role}. Nevertheless, the cost of interactions can become a bottleneck in designing adaptive algorithms~\citep{kairouz2021advances} and therefore designing algorithms that use a small number of adaptive rounds is crucial.


More recently, hypothesis selection has been studied under the constraint of differential privacy. The sample complexity of hypothesis selection is well understood in the central DP model~\citep{bun2019private, aden2021sample}, and it has a tight logarithmic dependence on $k$.
In the local model, however, 
there is still a gap between the best known upper and lower bounds for hypothesis selection as well as several other statistical learning tasks 
\citep{duchi2018minimax, duchi2019lower, joseph2019locally, bun2019heavy,acharya2020unified, chen2020breaking,    edmonds2020power,gopi2020locally, acharya2022role, asi2022optimal,asi2023fast}.
In this work, we fill this gap for the problem of hypothesis selection and characterize the optimal sample complexity in the local model. In particular, we propose a new iterative algorithm and a novel analysis technique, which together establish a \emph{linear} sample complexity in $k$ for LDP hypothesis selection. Before stating our results in detail, it is useful to define the problem of hypothesis selection formally.

\subsection{Hypothesis selection}
A hypothesis selector $\cA$ is a randomized function that receives a set of samples $S$, the description of a class of distribution $\cF$, accuracy and failure parameters $\alpha$ and $\beta$, and returns a distribution in $\cF$. Next, we give the formal definition of the hypothesis selection problem.
\begin{definition}[Hypothesis Selection]
    We say a hypothesis selector $\cA$ can do hypothesis selection with $m:\bN\times(0,1)^2 \rightarrow \bN$ samples if the following holds:
    for any unknown distribution $h$, any class $\cF$ of $k\in\bN$ distributions, and $\alpha,\beta>0$, if $S$ is a set of at least $m(k,\alpha,\beta)$ i.i.d.\ samples from $h$, then with probability at least $1-\beta$ (over the randomness of $S$ and $\cA$) $\cA(\cF,S,\alpha,\beta)$ returns $\hat{f}\in\cF$ such that $d_{TV}(h,\hat{f}) \leq C\min_{f\in\cF}d_{TV}(h,f) + \alpha$, where $C\geq 1$ is a universal constant and is called the approximation factor of $\cA$. The sample complexity of hypothesis selection is the minimum $m(k,\alpha,\beta)$ among all hypothesis selectors.
\end{definition}
Note that $C$ is a multiplicative \textit{approximation factor} and is typically a small constant. A smaller value for $C$ signifies a more accurate estimate $\hat f$.
The sample complexity of hypothesis selection is proportional to $\Theta\left(\frac{\log k}{\alpha^2}\right)$ as a function of $k$ and $\alpha$ (see~\citet{devroye2001combinatorial} for a detailed discussion). Other work has studied hypothesis selection with additional considerations, including computational efficiency, robustness, and more \citep{yatracos1985rates, mahalanabis2007density,acharya2014sorting,acharya2018maximum, bousquet2019optimal, bousquet2022statistically}.
Recently, \citet{bun2019private} showed that, similar to the non-private setting,  one can perform hypothesis selection under the constraint of central differential privacy with a logarithmic number of samples in $k$. 
However, the same statement does not hold in the local differential privacy setting. In fact, \citet{gopi2020locally}, building on a result of \citet{duchi2019lower} and \citet{braverman2016communication} showed that the sample complexity of locally private hypothesis selection problem scales \textit{at least linearly} in $k$. 
\begin{theorem}[Informal, Theorem~1.2 of \citet{gopi2020locally}, Corollary~6 of \citet{duchi2019lower}]\label{thm:gopi_LB}
      There exists a family of $k$ distributions for which any (interactive) $\eps$-LDP selection method requires at least $\Omega\left(\frac{k}{\alpha^2 \min\{\eps,\eps^2\}}\right)$ samples to learn it.
\end{theorem}

A basic idea to perform LDP hypothesis selection is using the classical tournament-based (round-robin) hypothesis selection method of \citet{devroye2001combinatorial}. While it is straightforward to come up with a locally private version of this approach, it would require $\Omega(k^2)$ samples as it compares all pairs of hypotheses. 
\citet{gopi2020locally} improved over this baseline and proposed a multi-round LDP hypothesis selection algorithm whose sample complexity scales as $O(k\log k \log \log k)$.
\begin{theorem}[Informal, Corollary~5.10 of \citet{gopi2020locally}]\label{thm:gopi}
     There exists an algorithm with failure probability of $1/10$, that solves the problem of hypothesis selection under the constraint of $\eps$-local differential privacy (for $\eps\in(0,1)$) with $O(\log \log k)$ rounds, with approximation factor of $27$, and using $O\left(\frac{k\log k { \log \log k}}{\alpha^2\min\{\eps^2,1\}}\right)$ samples.
\end{theorem}

The algorithm of \citet{gopi2020locally} is also a tournament-based method. However, unlike the round-robin approach, it only makes {$O(k\log\log k)$} comparisons between the hypotheses. This is achieved by comparing the distributions \emph{adaptively} in $\log \log k$ LDP rounds. A union bound argument is then required to make sure that every comparison is accurate, resulting in the $O(k \log k \log \log k)$ term in the sample complexity.
Given the gap between the upper bound of Theorem~\ref{thm:gopi} and the lower bound of Theorem~\ref{thm:gopi_LB} in \citet{gopi2020locally}, the following question remains open:

\begin{quote}
    Is it possible to perform hypothesis selection in the LDP model using $O(k)$ samples? If so, can it still be done in $O(\log \log k)$ adaptive LDP rounds?
\end{quote}
One barrier in answering the above question is the sample complexity lower bound of \citet{gopi2020locally} for \emph{non-interactive} hypothesis selection which is based on the work of~\citet{ullman2018tight}.

\begin{theorem}[Informal, Theorem~3.3 of \citet{gopi2020locally}]\label{thm:gopi_LB_NI}
      There is a family of $k$ distributions for which any non-interactive $\eps$-LDP hypothesis selection method requires at least $\Omega(\frac{k \log k}{\alpha^2\eps^2})$ samples.
\end{theorem}
In other words, we have to use interactivity to break the $O(k\log k)$ barrier on the sample complexity and achieve an $O(k)$ upper bound. This raises the following question:

\begin{quote}
 Does interactivity offer a provable sample-complexity benefit for locally private hypothesis selection?
\end{quote}

In this work, we show that the answers to the above questions are affirmative and, indeed, a small number of interactions helps us to achieve a linear sample complexity dependence on $k$.

\subsection{Results and discussion}
Our main result shows that locally private hypothesis selection can be solved with a sample complexity that is linear in $k$. To achieve this, we both design a \emph{new algorithm} and propose \emph{a new analysis technique}.
\begin{theorem}[Informal Version of Theorem~\ref{thm:ours}]\label{thm:ours_inf}
    There exists an $\varepsilon$-LDP algorithm that solves the problem of hypothesis selection in $O(\log\log k)$ rounds, with an approximation factor of $9$, and uses $O\left(\frac{k(\log 1/\beta)^2}{\alpha^2\min \{\eps^2,1\}}\right)$ samples.
\end{theorem}
Note that, unlike the existing upper bound given in Theorem~\ref{thm:gopi}, this result works for all values of $\beta$; a detailed discussion on this will follow.
The following corollary, which is a direct consequence of the above theorem and Theorem~\ref{thm:gopi_LB}, yields the optimal sample complexity. 
\begin{corollary}[Sample Complexity of LDP Hypothesis Selection]
Let $\eps\in(0,1)$. For any constant $\beta\in(0,1)$, the sample complexity of (interactive) $\eps$-LDP hypothesis selection is $\Theta\left(\frac{k}{\alpha^2\eps^2}\right)$.  
\end{corollary}
We now highlight some important aspects of our results and contributions.

{\bf Optimal dependence on $k$.} 
The private hypothesis selector proposed in \cite{gopi2020locally} relies on $\Theta(\log\log k)$ adaptive rounds and requires {$O(k\log\log k)$} pairwise comparisons. Applying the union bound then yields the sub-optimal sample complexity that scales with $k\log k \log \log k$. On the other hand, we develop a new analysis technique that avoids the simple union bound over all comparisons. Instead, it exploits the fact that not all the comparisons are \textit{critical} for establishing the guarantees of an algorithm (See Section~\ref{sec:critical}). We also propose a new algorithm that {runs in the same number of rounds as the algorithm in \citet{gopi2020locally} but uses only $O(k)$ queries, from which only {$O(k/\log k)$} are {critical} queries (See Section~\ref{sec:algorithms}).}

 {\bf Approximation factor.} Theorem~\ref{thm:ours_inf} solves the hypothesis selection with an approximation factor of $9$, compared to the factor of $27$ in Theorem~\ref{thm:gopi}. We achieve this by using a variant of Minimum Distance Estimate \citep{mahalanabis2007density} as the final sub-routine in our algorithm (see Section~\ref{sec:algorithms}).



{\bf High probability bound.} A drawback of the upper bound in \citet{gopi2020locally} is its loose dependence on the failure parameter $\beta$. In fact, Theorem~\ref{thm:gopi} is stated only for $\beta=1/10$. 
In contrast, our result holds for any value of $\beta>0$ with a mild cost of $(\log1/\beta)^2$ in the sample complexity. We achieve this by boosting the success probability of our algorithm in various steps\footnote{It is possible extend the analysis of \citet{gopi2020locally} to make it work for any $\beta\in (0,1)$ but with with a substantial $(1/\beta)^2$ cost in the sample complexity (as opposed to the poly-logarithmic dependence in our bound).}.

{\bf Computational complexity.} {Similar to \citet{gopi2020locally}, our algorithm runs in linear time (as a function of the number of samples) assuming an oracle access to the Scheff\'e sets of distributions.} 

{\bf Number of rounds.} A main feature of the algorithm of~\citet{gopi2020locally} is that it runs in only $\Theta(\log\log k)$ adaptive rounds. This sets it apart from similar approaches~\citep{acharya2014sorting, acharya2018maximum} that require $\Theta(\log k)$ adaptive rounds. Our method also runs in $\Theta(\log\log k)$ rounds. 
Moreover, our result demonstrates the power of interaction: any non-adaptive LDP method for hypothesis selection requires at least $\Omega(k\log k)$ samples~\citep{ullman2018tight, gopi2020locally}, while our algorithm works with $O(k)$ samples.
It remains open whether a sample-optimal LDP algorithm can run in $o(\log\log k)$ rounds.

{\bf The statistical query viewpoint.}
Our analysis is presented in the more general Statistical Query (SQ) model \citep{kearns1998efficient}. SQ algorithms can be readily implemented in the LDP model~\citep{kasiviswanathan2011can}; see Section~\ref{sec:ldphs}. 
Moreover, most (if not all) existing hypothesis selection methods are SQ algorithms (and therefore can be implemented in the LDP model). We provide a summary of existing hypothesis selection methods and the cost of implementing them in the LDP model in~Table~\ref{table:summary}.
\begin{table}
\small
\begin{center}
\begin{tabular}{ |c|c|c|c|c|} 
 \hline
 Hypothesis selection method & App. factor & \#Queries & \#Rounds & \#Samples for LDP\\
 \hline
 \begin{tabular}{@{}c@{}}Round-Robin\\ \citep{devroye2001combinatorial}\end{tabular} & $9$ & $O(k^2)$ & $1$ & $O(k^2\log k)$\\ 
  \hline  
  {{Comb}}~\citep{acharya2014sorting} & 9 & $O(k)$ & $O(\log k)$ & $O(k\log k)$ \\ 
  \hline
  \citet{gopi2020locally} & $27$ & $O(k\log \log k)$ & $O(\log\log k)$ & $O(k\log k \log \log k)$\\ 
  \hline
  \citet{gopi2020locally}, $\forall t\in[k]$& $9^t$ & $O(k^{1+\frac{1}{2^t-1}}t)$ & $t$ & $O(k^{1+\frac{1}{2^t-1}}t\log k)$\\ 
  \hline
     \begin{tabular}{@{}c@{}}MDE-Variant\\ \citep{mahalanabis2007density}\end{tabular} & $3$ & $O(k^2)$ & $1$ & $O(k^2\log k)$\\ 
  \hline
    \textsc{BOKSERR [This work]} & $9$ & $O(k)$ & $O(\log \log k)$ & $O(k)$\\ 
  \hline
\end{tabular}
\caption{Hypothesis selection methods, their approximation factors, number of statistical queries they ask, number of rounds and samples required to implement them in the LDP model.}\label{table:summary}
\end{center}
\end{table}
\subsection{Related work}
Hypothesis selection is a classical problem in statistics~\citep{scheffe1947useful,yatracos1985rates, devroye1996universally,devroye1997nonasymptotic, devroye2001combinatorial}; see \cite{devroye2001combinatorial} for an overview. More recent papers have studied hypothesis selection under various considerations such as computational efficiency, robustness and more~\citep{mahalanabis2007density, daskalakis2012learning, daskalakis2014faster, suresh2014near, acharya2014sorting, acharya2018maximum, diakonikolas2019robust, bousquet2019optimal, bousquet2022statistically}.
Other related problems are hypothesis testing~\citep{neyman1933ix, paninski2008coincidence, batu2000testing, goldreich2011testing,diakonikolas2016new, gs009} and distribution learning~\citep{yatracos1985rates, devroye2001combinatorial, birge1989grenader, hasminskii1990density, yang1999information, devroye2001combinatorial, moitra2010settling, chan2014efficient, acharya2015fast, diakonikolas2016learning, acharya2017sample, ashtiani2018some, ashtiani2020near, diakonikolas2023algorithmic}.

Differential privacy was introduced in the seminal work of \citet{dwork2006calibrating}. The central models of DP (including the pure~\citep{dwork2006calibrating} and approximate~\citep{dwork2006our} models) have been studied extensively. In the central model, the first private hypothesis selection method was proposed by \citet{bun2019private}. This result was later improved by \cite{aden2021sample}. Other problems that have been investigated in the central model are distribution testing \cite{canonne2019structure, canonne2020private, narayanan2022private} and distribution learning \cite{karwa2018finite, kamath2019differentially, kamath2019privately,bun2019private, biswas2020coinpress, kamath2020private, aden2021sample, aden2021privately, kamath2022private, ashtiani2022private, kothari2022private, alabi2023privately,hopkins2023robustness, cohen2021differentially, tsfadia2022friendlycore, bie2022private, singhal2023polynomial, arbas2023polynomial, ben2023private, afzali2023mixtures}.

The local model of DP~\cite{warner1965randomized, evfimievski2003limiting, kasiviswanathan2011can, duchi2013local} is more stringent than the central one, and has been the model of choice for various real-world applications~\citep{erlingsson2014rappor,apple,ding2017collecting,patent}. Several lower and upper bounds have been established for various statistical estimation tasks under local DP \citep{duchi2018minimax, joseph2019locally, duchi2019lower, bun2019heavy, chen2020breaking, asi2022optimal, asi2023fast, acharya2021inference, acharya2020unified, asoodeh2022contraction}, including hypothesis testing~\citep{pensia2022simple, pensia2023simple, acharya2019test} and selection~\citep{gopi2020locally}.


Studying statistical estimation tasks in the statistical query model is another related topic~\citep{kearns1998efficient, szorenyi2009characterizing, feldman2017general, diakonikolas2017statistical}.
More related to our work are those that study the sample complexity of answering multiple statistical queries under differential privacy \citep{ullman2013answering,steinke2016between, edmonds2020power, ghazi2021avoiding, dagan2022bounded}.

The role of interaction in local models has been investigated in several works \citep{kasiviswanathan2011can, duchi2018minimax, duchi2019lower, joseph2019role, daniely2019locally, joseph2020exponential} and it has been shown that it is a powerful tool in designing algorithms that are sample efficient~\citep{han2018geometric, acharya2020unified, gopi2020locally, acharya2022role}. We also prove in this work that using adaptive algorithms can solve the task of hypothesis selection more efficiently.  
\section{Preliminaries}
{\bf Notation.} We denote by $[N]$ the set of numbers $\{1,2,\ldots, N\}$. The total variation distance between two probability densities $f$ and $g$ over domain $\cX$ is defined by $\allowdisplaybreaks d_{TV}(f,g)=\frac{1}{2}\int_{\cX}\left|f(x)-g(x)\right|dx = \frac{1}{2}\|f-g\|_1$. We define the total variation distance between a class of distributions $\cF$ and a distribution $h$ as $d_{TV}(h,\cF)=\inf_{f\in\cF}d_{TV}(h,f)$. We abuse the notation and  use $f$ both as a probability measure (e.g., in $f[B]$ where $B\subset \cX$ is a measurable set) and as its corresponding probability density function (e.g., in $f(x)$ where $x\in\cX$).

\subsection{The statistical query viewpoint of hypothesis selection}\label{sec:sq_model}

Assume $h$ is an (unknown) probability distribution over domain $\cX$. In a typical scenario, the learning algorithm is assumed to have direct access to random samples from $h$. Many learning algorithms, however, can be implemented in the more limited \emph{Statistical Query (SQ)} model~\citep{kearns1998efficient}. 
Namely, instead of accessing random samples, the learning algorithm chooses a (measurable) function $q:\cX\rightarrow \{0,1\}$ and receives an estimate of $\expects{x\sim h}{q(x)}$.

The hypothesis selection algorithms that we study in this paper are based on pair-wise comparisons between distributions. In Section~\ref{sec:existing}, we show that such a comparison can be executed by a statistical query. 
Therefore, it is helpful to view these hypothesis selection algorithms as Statistical Query Algorithms (SQAs). We define a Statistical Query Oracle (SQO) and a SQA in below.
 

\begin{definition}[Statistical Query Oracle]\label{def:SQO}
    Let $(\cX, \mathcal{B}, h)$ be a probability space. A Statistical Query Oracle (SQO) for $h$ is a random function $\cO_h$ with the following property: for any $\alpha,\beta \in (0,1)$, and a finite workload $W=(W_i)_{i=1}^n$ where $W_i \in \cB$, the oracle outputs $n$ real values such that
    \begin{equation*}
        \prob{\sup_{i\in [n]}\left| \cO_{h}\left(W,\alpha,\beta \right)_i-\expects{x\sim h}{\indicator{x\in W_i}}\right|\geq \alpha} \leq \beta,
    \end{equation*}    
    where $\cO_{h}\left(W,\alpha,\beta \right)_i$  is the $i$-th output of $\cO_{h}$ and the probability is over the randomness of oracle.
    \end{definition}

Based on the above definition, we define the statistical query algorithm that runs in $t$ rounds and calls a SQO in each round adaptively (i.e., based on the outcomes of previous rounds).
\begin{definition}[Statistical Query Algorithm]\label{def:SQA}
    We say $\cA$ is a Statistical Query Algorithm (SQA) with $t$ rounds if it returns its output by making $t$ (adaptive) calls to a statistical query oracle $\cO_h$, where in each call $1<i\leq t$, the workload $W^{(i)}$ of queries may depend on the output of $\cO_h$ in the previous round $i-1$ rounds. 
\end{definition}    

\subsection{Locally private hypothesis selection}\label{sec:ldphs}
In the local model of DP, each data point goes through a private mechanism (called a local randomizer) and only the privatized outcomes are used by the algorithm. A canonical example of a local randomizer is the composition of the randomized response mechanism with a binary function. The randomized response is a mechanism that gets as input a bit and outputs it with probability $p$ or flips it with probability $1-p$. An LDP algorithm operates in multiple rounds. In each round, the algorithm chooses a series of local randomizers based on the information that it has received in the previous rounds. Each selected randomizer then generates a noisy version of a new data point for the use of algorithm.
The local privacy model is defined more formally in Section~\ref{sec:ldp}.

The hypothesis selection algorithms discussed in this paper all fit into the SQA framework described in Definition~\ref{def:SQA}. We show in Section~\ref{sec:ldp} that any SQA can be implemented within the local privacy model using the same number of rounds.
We also discuss how many samples are required to simulate a SQO with a single-round LDP protocol. It turns out that in general $\Theta(\frac{n\log n}{\alpha^2\min\{\eps^2,1\}})$ samples are sufficient 
to simulate a SQO that answers a workload of $n$ queries with accuracy $\alpha$ and failure $\beta=2/3$ using a single-round LDP method. This allows us to simplify the presentation of the results by focusing on the query complexity of the hypothesis selection approaches. 


We find it useful to give an informal discussion on how a SQO can be implemented with a single round LDP algorithm. A natural way of answering a SQ is to return the empirical average of the query values on a set of i.i.d.\ samples. However, to satisfy the privacy constraint, the algorithm cannot use the query values on the actual samples. Instead, the algorithm can observe the private version of these values through a (de-biased) randomized response mechanism. An application of Hoeffding's inequality implies that $\Theta(\frac{\log 1/\beta}{\alpha^2\min\{\eps^2,1\}})$ samples are sufficient to make sure that the empirical average of the randommized response values is $\alpha$-close to the true value of query with probability at least $1-\beta$. To ensure $\alpha$ accuracy for all queries with probability at least $1-\beta$, one needs to resort to a union-bound argument over all $n$ queries. Therefore, each query needs to be estimated with a higher confidence of $1-\beta/n$ (i.e., failure of $\beta/n$), which results in the factor of $O(\log n)$ in sample complexity. It is important to note that, due to the privacy constraint, we need to use fresh samples for answering each query (otherwise, we incur a larger privacy budget).
 Hence, the sample complexity of this approach is $\Theta(\frac{n\log n/\beta}{\alpha^2\min\{\eps^2,1\}})$. In Section~\ref{sec:critical}, we introduce the notion of \emph{critical} queries and show how one can avoid the typical union bound analysis and, thus, the sub-optimal logarithmic term in the sample complexity. In Section~\ref{sec:ldp}, we expand on this and show how to privately implement any SQA.

\subsection{Existing tournament-based approaches for hypothesis selection}\label{sec:existing}

In this section, we review some of the classical algorithms for hypothesis selection, and how they can be implemented as SQAs. As discussed in the previous section, any SQA can be implemented in the LDP setting too.
Some of these algorithms will be used later as building blocks of our method. 
We relegate the pseudo-codes for these algorithms to Appendix~\ref{app:algo}.

\subsubsection{Classes of two distributions} 
We begin with the simplest case where $\cF=\{f_1,f_2\}$ and the classical Shceff\'e test \citep{scheffe1947useful, devroye2001combinatorial}. This algorithm is a building block of other algorithms in the sequel. Let us first define the Scheff\'e set between two distributions.
\begin{definition}[Scheff\'e Set]
    The Scheff\'e set between an ordered pair of distributions $(f,g)$ is defined as $ Sch(f,g) = \{x\in\cX: f(x)>g(x)\}$.
\end{definition}

Let $h$ be the (unknown) probability measure that generates the samples. The Shceff\'e test, delineated in Algorithm~\ref{alg:scheffe}, chooses between $f_1$ and $f_2$ by estimating $h[Sch(f_1,f_2)]$. Note that $h[Sch(f_1,f_2)]$ can be estimated using a single statistical query.
The following folklore theorem states the formal guarantee of the Shceff\'e test.

\begin{theorem}[Analysis of Shceff\'e Test]\label{thm:scheffe}
Let $\cF=\{f_1,f_2\}$, $\alpha,\beta\in(0,1)$ and $\cO_h$ be a SQO for the (unknown) distribution $h$. Let $y = \cO_h(Sch(f_1,f_2),\alpha,\beta)$. If $\hat{f} =\textsc{SCHEFF\'E}(f_1,f_2,y)$, then we have $d_{TV}(h,\hat{f}) \leq 3d_{TV}(h,\cF) + \alpha$ with probability at least $1-\beta$.
\end{theorem}

\subsubsection{Classes of $k$ distributions}

In the more general setting where $|\cF|=k$, one possible approach is to simply run a Scheff\'e test for all $\Theta(k^2)$ pairs of distributions and output the distribution that is returned the most. The round-robin tournament \citep{devroye2001combinatorial} implements this idea (see see Algorithm~\ref{alg:rr}), and achieves an approximation factor of 9. The following establishes the guarantee of this method; see Theorem~6.2 in \cite{devroye2001combinatorial} for a proof.

\begin{theorem}[Analysis of Round-Robin]\label{thm:rr}
Let $\cF=\{f_1,\ldots,f_k\}$ be a set of $k$ distributions, $\alpha,\beta\in(0,1)$, and $\cO_h$ be a SQO for the (unknown) distribution $h$. Algorithm $\textsc{round-robin}(\cF,\cO_h,\alpha,\beta)$ is a SQA with a single round that makes at most $\frac{k(k-1)}{2}$ queries with accuracy $\alpha$ to $\cO_{h}$, 
and returns a distribution $\hat{f}$ such that  $d_{TV}(h,\hat{f})\leq 9 d_{TV}(h,\cF) + \alpha$ with probability at least $1-\beta$.
\end{theorem}
{\bf Improving the approximation factor.} It is possible to improve over the approximation factor of round-robin by using the MDE-variant algorithm \citep{mahalanabis2007density} described in Algorithm~\ref{alg:mde}. This algorithm is an ``advanced'' version of the well-known Minimum Distance Estimate (MDE) algorithm \citep{yatracos1985rates} that solves the hypothesis selection problem with the approximation factor of $3$ and $\Theta(k^2)$ queries. This is significantly better than the classic MDE that requires $\Theta(k^3)$ queries. Similar to round-robin, MDE-variant also relies on pair-wise comparisons.
\footnote{MDE and MDE-variant can also be used in the setting where $|\cF|$ is unbounded. In this case, the sample complexity will depend on the VC dimension of the Yatracose class. See \citet{devroye2001combinatorial} for details.} The precise guarantee of MDE-variant is given in the following theorem. The proof can be found in \citet{mahalanabis2007density}.

\begin{theorem}[Analysis of MDE-Variant]\label{thm:mde}
   Let $\cF=\{f_1,\ldots,f_k\}$ be a set of $k$ distributions and let $\alpha,\beta\in(0,1)$.  Let $\cO_h$ be SQO for the (unknown) distribution $h$. Algorithm $\textsc{MDE-variant}(\cF,\cO_h,\alpha,\beta)$ is a SQA that runs in a single round and returns a distribution $\hat{f}\in\cF$ such that $d_{TV}(h,\hat{f})\leq 3 d_{TV}(h,\cF) + \alpha$ with probability at least $1-\beta$. Moreover, it makes $\frac{k(k-1)}{2}$ queries of accuracy $\alpha$ to the oracle.
\end{theorem}

Both the round-robin and MDE-variant are non-adaptive (i.e., run in a single round) and make $\Theta(k^2)$ queries. Essentially, these algorithms query about every possible Scheff\'e set between pairs of distributions. This may not be a significant issue in settings where answering many queries does not incur huge cost, e.g., in the non-private setting where we can reuse the samples. However, as mentioned in Section~\ref{sec:ldphs}, the number of samples under LDP constraint scales (roughly) linearly with the number of queries. Hence, simply simulating round-robin or MDE-variant in the local model results in a sub-optimal (with respect to $k$) sample complexity bound of $O(k^2)$. 
In the next section, we review adaptive methods and how they help to reduce the number of required queries for LDP hypothesis selection.

\subsubsection{Hypothesis selection with adaptive statistical queries} \label{sec:gopi}
\citet{gopi2020locally} developed an adaptive algorithm that only makes {$O(k\log\log k)$} queries to solve hypothesis selection. The algorithm partitions the candidate distributions in several groups with small size 
and runs a round-robin in each group. It then moves the winners to the next round and eliminates the rest of the candidates. The algorithm runs in $O(\log \log k)$ rounds and the size of groups increases in each round. Running the round-robin tournaments on groups with small size makes {it possible to reduce the number of queries from $O(k^2)$ to $O(k\log\log k)$.} However, the sequential nature of this method causes the resulting approximation factor to grow exponentially with the number of rounds. To address this issue, they proposed a modification of this algorithm using a general tool that we will discuss in detail in Section~\ref{sec:algorithms}. They then showed that the resulting approximation factor becomes $27$ for the modified algorithm (see Theorem~\ref{thm:gopi}).   



We should mention that the hypothesis selection method of \cite{acharya2014sorting} can also be implemented using $\Theta(k)$ queries. However, compared to \citet{gopi2020locally}, it uses a significantly larger number of interactions (i.e., it runs in $\Theta(\log k)$ round rather than $\Theta(\log \log k)$).

As mentioned in Section~\ref{sec:ldphs}, it is possible to privately implement a SQA that makes $O(k)$ queries using $O(\frac{k\log k/\beta}{\alpha^2\eps^2})$ samples. Such implementation was a main technical tool in the sample complexity analysis of Theorem~\ref{thm:gopi}. Note that the $O(\log k)$ term in the sample complexity comes from taking a union-bound argument over the accuracy of all the $O(k)$ queries. In the next section, we present the framework of \emph{critical queries} which can help to get a better sample complexity.

\section{SQ model with critical queries}\label{sec:critical}
In this section, we present a new framework for the sample complexity analysis of answering multiple statistical queries.  As stated earlier, it is possible to implement an oracle that $\alpha$-accurately answers $n$ queries with probability at least $1-\beta$ using $O(\frac{n\log n/\beta}{\alpha^2\eps^2})$ samples. The fact that the oracle is required to be accurate in \textit{all} the queries results in the $O(\log n)$ factor in the sample complexity (see Section~\ref{sec:ldphs}). But what if the algorithm can establish its guarantees without the need to be \emph{confident} about the accuracy of all queries?
In other words, can we improve the sample complexity if we know that the success of the algorithm hinges on the correctness of only a small subset of the queries (i.e., \emph{critical} queries)?
To formulate this intuition, we define the following. 

\begin{definition}[Statistical Query Oracle with Critical Queries]\label{def:critical}
    Let $(\cX, \mathcal{B}, h)$ be a probability space. A Statistical Query Oracle with Critical queries (SQOC) for $h$ is a random function that receives $\alpha,\beta \in (0,1)$, a finite workload $W=(W_i)_{i=1}^n$ where $W_i \in \cB$, and $m\in [n]$, and outputs $n$ real values. Formally, we say $\cO_{h}$ is an SQOC for $h$ if for all $\alpha,\beta\in (0,1), n\in \bN, i\in [n], W=(W_1,\ldots, W_n)\in \cB^n$ we have  
    \begin{equation*}
        \forall U\subset [n], |U|=m,\,\prob{\sup_{i\in U}\left| \cO_{h}\left(W,\alpha,\beta,m \right)_i-\expects{x\sim h}{\indicator{x\in W_i}}\right|\geq \alpha} \leq \beta,
    \end{equation*}    
    where $\cO_{h}\left(W,\alpha,\beta,m \right)_i$ is the $i$-th output of $\cO_{h}$ and the probability is over the randomness of oracle (for the fixed choice of input and $U$). We refer to $m$ as the number of critical queries and to $\alpha,\beta$ as the approximation accuracy and the failure probability, respectively. 

\end{definition}

Compared to an SQO, SQOC receives an additional input of $m$ that indicates the number of critical queries. An SQOC offers a weaker guarantee than a standard SQO (Definition~\ref{def:SQO}) in that it ensures $\alpha$-accuracy only for subsets of size $\leq m$ of (out of $n$) queries. That is, Definition~\ref{def:critical} reduces to Definition~\ref{def:SQO} only for $m = n$. 
Note that, crucially, an algorithm that uses an SQOC does not need to know which specific queries are critical for its success (otherwise, it would have asked just those queries). Instead, the algorithm only needs to know (a bound on) the number of such critical queries. 
The potential benefit of using an SQOC over an SQO is that it can be implemented in the LDP setting using $O(\frac{n\log m}{\alpha^2\min\{\eps^2,1\}})$ samples (See Observation~\ref{obs} in Section~\ref{sec:ldp}).

\subsection{Critical queries for maximal selection with adversarial comparators}

In this section, we motivate the concept of critical queries further by providing a simple example in which critical queries can help improving the analysis.
We first describe (and slightly modify) the setting of maximal selection with adversarial comparators \citep{acharya2014sorting}. 
Assume we have a set of $k$ items $X=\{x_1,\ldots,x_k\}$ and a value function $V:X \rightarrow \bR$ that maps each item $x_i$ to a value $V(x_i)$. Our goal is to find the item with the largest value by making pairwise comparisons. As an intermediate step, we want to pair the items randomly, compare them, and eliminate half of them. We, however, do not know the value function and only have access to an adversarial comparator $\cC_V:X^2\rightarrow X$. The comparator gets as input two items $x_i$ and $x_j$, consumes $m(\alpha,\beta)$ fresh samples and \emph{(i)} if $|V(x_i)-V(x_j)|>\alpha$ then with probability at least $1-\beta$ outputs $\arg\max_{l\in\{i,j\}}\{V(x_l)\}$ and \emph{(ii)} if $|V(x_i)-V(x_j)|\leq \alpha$ then returns either item randomly. Here, $m:(0,1)^2\rightarrow \bN$ is decreasing in $\alpha$ and $\beta$, i.e., higher accuracy and smaller failure probability requires more samples. We now define the problem formally. 

\begin{example}
    Let $X$ be a set of $k$ items, $V:X\rightarrow \bR$ be a value function, and $\cC_V:X^2\rightarrow X$ be an adversarial comparator for the value function, and $\alpha,\beta\in(0,1)$. Assume $X^*=\{x\in X: |V(x)-\max_{x^*\in X} V(x^*)|\leq\alpha\}$ is the set of items with values comparable to the maximum value up to error $\alpha$. Assume there is an algorithm that pairs the items randomly and invokes $\cC_{V}$ to compare them. The algorithm returns the set $Y$ of the $k/2$ winners.
    How many samples do we need to make sure that with probability at least $1-\beta$ we have a ``good'' item in $Y$, i.e., $X^*\cap Y \neq \emptyset$?
\end{example}
{\bf Approach 1.}  Recall that if we have $m(\alpha,\beta)$ samples, then any pair of items $(x,x')$ can be compared with accuracy $\alpha$ and with probability at least $1-\beta$. Applying union bound, we deduce that the algorithm can correctly compare all $k/2$ pairs with accuracy $\alpha$ if it has access to $\frac{k}{2}m(\alpha,\frac{2}{k}\beta)$ samples. Since the comparator is $\alpha$-accurate, it follows that either a member $x^* \in X^*$ is included in $Y$ or it has lost the comparison to another item $\hat{x}$ with $|V(\hat{x})-V(x^*)|\leq\alpha$, which implies that $x'\in X^*$. Either way, $X^*\cap Y \neq \emptyset$ with probability at least $1-\beta$.

\noindent {\bf Approach 2 (Critical query).} 
Let $x^*\in X^*$ and  $x'$ be the item paired with $x^*$. Clearly, if $x^*$ is compared accurately, then $X^*\cap Y \neq \emptyset$. Thus, unlike the previous approach, we only need to inspect a single comparison. In other words, whether two other items $x_i$ and $x_j$ are compared accurately can be ignored in analyzing $X^*\cap Y\neq \emptyset$. This shows that we need only $\frac{k}{2}m(\alpha,\beta)$ samples, which is smaller than what we established with the previous approach. Notice that although the algorithm is unaware of the comparison involving $x^*$, there is only one critical query for the desired task.  

The above example demonstrates that the concept of critical queries enables us to derive the same guarantees but with a smaller sample complexity, compared to what would be obtained simply by union bound.
\section{A Sample-optimal algorithm for hypothesis selection}\label{sec:algorithms}


The algorithm of \citet{gopi2020locally}, basically, has as many critical queries as the total number of queries. Thus, using the concept of critical queries does not lead to a better sample complexity for their approach.
In this section, we propose a new locally private hypothesis selector (which we term BOKSERR) that uses a smaller number of critical queries and achieves the optimal (i.e., linear) sample complexity dependence on $k$. 
BOKSERR consists of three sub-routines, namely boosted knockout, boosted sequential round-robin, and MDE-variant. All the algorithms can be found in Appendices~\ref{app:knockout}, \ref{app:SRR}, and \ref{app:algo}, respectively. Before discussing each sub-routine, it is useful to introduce the following design technique that is borrowed from \citet{acharya2014sorting}.

{\bf A design technique and some notations.}
Let $f^*\in\cF$ be such that $d_{TV}(h,f^*) = d_{TV}(h,\cF)$ and $f\in\cF$ be another distribution in $\cF$. Theorem~\ref{thm:scheffe} suggests that if $d_{TV}(h,f)> 3d_{TV}(h,\cF)+\alpha$, then with high probability, $\textsc{SCHEFF\'E}(f,f^*)$ returns $f^*$. Assume $\allowdisplaybreaks S = \{f\in \cF:d_{TV}(h,f)\leq 3d_{TV}(h,\cF)+\alpha\}$ and $\zeta=|S|/|\cF|$. Intuitively, if $\zeta$ is small, then $f^*$ is likely to win when compared to a randomly chosen distribution from $\cF$. Otherwise, if $\zeta$ is large, then any appropriately-sized random sub-sample of $\cF$ will contain a distribution from $S$ with high probability. This intuition has been incorporated in our algorithm design: The first two sub-routines (namely, boosted knockout and boosted sequential round-robin) return two lists of distributions 
$\cL_1$ and $\cL_2$ such that with high probability either $f^*\in\cL_1$ or $S\cap\cL_2\neq\emptyset$. Put differently, there exists a ``good'' distribution in $\cL_1\cup\cL_2$ with high probability. The proof appears in Appendix~\ref{app:knockout}.


\subsection{Boosted knockout} Given $\cF$, oracle $\cO_h$, and $t\geq 1$, boosted knockout runs in $t$ rounds and returns two lists of distributions. The first one is constructed as follows. In each round, it randomly pairs distributions in $\cF$ for $r=O(\log 1/\beta)$ times and makes pairwise comparisons, where each comparison is carried out by a Scheff\'e test. It then selects the distributions that win at least $\frac{3}{4}$ of the comparisons in which they are involved and moves them to the next round. This ensures that the number of candidate distributions is reduced by at least a factor of $\frac{3}{2}$ in each round (see Appendix~\ref{app:knockout} for a proof). The distributions that have progressed to the last round constitute the first list. The second list is a random sub-sample of size $O(2^t\log(1/\beta))$ from the initial set of distributions $\cF$.
As mentioned before, the union of the above two lists includes at least one distribution from $S$, as formalized in the following theorem. 

\begin{theorem}[Analysis of Algorithm~\ref{alg:ko}]\label{thm:knockout}
   Let $\cF=\{f_1,\ldots,f_k\}$ be a set of $k$ distributions, and let $\alpha,\beta\in(0,1)$ and $t\geq 1$. 
   Let $\cO_h$ be a SQOC for the (unknown) distribution $h$.
   Algorithm $\textsc{boosted-knockout}(\cF,\cO_h,\alpha,\beta,t)$ is a SQA that runs in $t$ rounds. The number of statistical queries (i.e., size of the workload) to $\cO_h$ in round $i\in[t]$ is at most $\frac{64}{3}(\frac{8}{9})^i k\log{\frac{1}{\beta}}$ with $32(\frac{4}{3})^i\log\frac{1}{\beta}$ many critical queries. 
   Moreover, it returns two lists of distributions $\cK_1$ and $\cK_2$ with $|\cK_1|\leq\frac{k}{2^{t.\log\frac{3}{2}}}$ and $|\cK_2|\leq 8\log\frac{1}{\beta}2^{t.\log\frac{3}{2}}$ such that with probability at least $1-\beta$ either $d_{TV}(h,\cK_1)=d_{TV}(h,\cF)$ or $d_{TV}(h,\cK_2)\leq 3 d_{TV}(h,\cF)+\alpha$.


\end{theorem}

We highlight two key properties of boosted knockout. First, it makes a small number of critical queries in each round, which are only those that correspond to comparing $f^*$ with other distributions. To see this, note that we bound the probability that $f^*$ is not in the first list by the probability that it will be paired with a distributions in $S$ for at least $1/4$ of its comparisons.
The analysis does not depend on the result of other comparisons. This implies that in each round only the $r$ comparisons that concern $f^*$ are critical.  Second, it eliminates a subset of candidate distributions and prepares a smaller list of distributions for the next sub-routine sequential-round-robin. A more careful analysis of boosted knockout is given in the next theorem.

\subsection{Boosted sequential round-robin}
We now discuss the second sub-routine, namely, boosted sequential-round-robin (BSRR for short) described in Algorithm~\ref{alg:srr}. Similar to boosted knockout, this algorithm relies on sequentially reducing the list of potential candidates starting from the $\cF$, and eventually returning two lists of distributions. The first list is determined by sequential (i.e., adaptive) comparisons between distributions, while the second list is an appropriately-sized sub-sample from $\cF$. However, instead of pairing the distributions, BSRR partitions the candidate distributions into a number of groups in each round, runs a round-robin tournament in each group, and keeps only the winners for the next round. The size of the groups are squarred in each round of BSRR and, therefore, the size of the candidate distributions decreases very quickly. 
If $\zeta=|S|/|\cF|$ is small then with high probability $f^*$ is not grouped with any distribution from $S$ and will not be eliminated until the very last round. If $\zeta$ is large, then a distribution in $S$ will be included in a random sub-sample of $\cF$ with high probability. Hence, the union of the two lists will include a distribution from $S$ with high probability.

BSRR is adapted from \citet{gopi2020locally} with one main difference. Similar to boosted knockout, instead of only partitioning the distributions once in every round, we repeat this process for $O(\log 1/\beta)$ times and keep the winners of all of the round-robin tournaments. This results in boosting the probability that, in each round, there exists at least one group that includes $f^*$ but not any other distributions from $S$. Thus, the probability that $f^*$ is included in the first list also increases. The guarantee of BSRR is characterized in the following. In Appendix~\ref{app:SRR}, 
We show why these repetitions are necessary to achieve the logarithmic dependency on $1/\beta$ in the sample complexity and provide a proof for the theorem.

\begin{theorem}[Analysis of Algorithm~\ref{alg:srr}]\label{thm:srr}
   Let $\cF=\{f_1,\ldots,f_k\}$ be a set of $k$ distributions, and let $\alpha,\beta\in(0,1)$, $\eta>0$, and $t\geq 1$. Let $\cO_h$ be a SQOC for the (unknown) distribution $h$. Algorithm $\textsc{Boosted-Sequential-round-robin}(\cF,\cO_h,\alpha,\beta,\eta,t)$  is a SQA that runs in $t$ rounds.
   The number of statistical queries (i.e., size of the workload) to $\cO_h$ in round $i\in[t]$ is at most $k\eta \left(\log\frac{1}{\beta}\right) ^ {i}$, all of which are critical. The algorithm returns two lists of distributions $\cR_1$ and $\cR_2$ with $|\cR_1|\leq \frac{k(\log\frac{1}{\beta})^{t}}{\eta^{2^{t}-1}}$ and $|\cR_2| \leq 2\eta^{2^t}\log\frac{1}{\beta}$ such that with probability at least $1-\beta$ either $d_{TV}(h,\cR_1)=d_{TV}(h,\cF)$ or $d_{TV}(h,\cR_2)\leq3 d_{TV}(h,\cF)+\alpha$. 
\end{theorem}

\subsection{Boosted-sequential-round-robin-MDE-variant (BOKSERR)}
We are now ready to delineate the overall algorithm, that we call boosted-sequential-round-robin-MDE-variant (BOKSERR for short). It starts by calling the boosted knockout with $\Theta(\log\log k)$ rounds to construct two lists of distributions $\cK_1$ and $\cK_2$. The set $\cK_1$ is then fed to BSRR with the same $\Theta(\log\log k)$ number of rounds to further generate two lists of distributions $\cR_1$ and $\cR_2$. Recall that all of the queries made by BSRR are critical. However, knockout ensures that the size of the input list to BSRR, namely $\cK_1$, is $O(\frac{k}{(\log k)^{\log\log 1/\beta}})$. This makes sure that that the number of queries that BSRR makes is small and Observation~\ref{obs} suggests that the queries can be answered with a sample complexity that is sub-linear in $k$.

{Finally, BOKSERR uses MDE-variant (that was discussed in Section~\ref{sec:existing}) to select its output distribution from $\cR_1\cup\cR_2\cup\cK_2$. 
The union of the three lists is of size $O\left(\frac{\sqrt{k}(\log\frac{1}{\beta})}{(\log k)^{\log\log \frac{1}{\beta}}}\right)$, and therefore MDE-variant sub-routine (similar to round-robin) needs $O(k\log^2 \frac{1}{\beta})$ samples to choose the output distribution from this lists.} The guarantee of BOKSERR is given as follows.  

\begin{theorem}[Analysis of Algorithm~\ref{alg:ours}]\label{thm:ksrm}
   Let $\cF=\{f_1,\ldots,f_k\}$ be a set of $k$ distributions and let $\alpha,\beta\in(0,1)$.  Let $\cO_h$ be SQOC for the (unknown) distribution $h$. Algorithm $\textsc{BOKSERR}(\cF,\cO_h,\alpha,\beta)$ is a SQA that runs in $\left(6+4\log\log\frac{3}{\beta}\right)\log\log k$ rounds and returns a distribution $\hat{f}\in\cF$ such that with probability at least $1-\beta$ we have $d_{TV}(h,\hat{f}) \leq 9d_{TV}(h,\cF)+\alpha$. 
   The total number of statistical queries is $\Theta(k\log^2\frac{1}{\beta})$, and the number of critical queries is  ${O\left( \frac{k\log^2\frac{1}{\beta}}{(\log k)^{\log\log \frac{1}{\beta}}}\right)}$.
\end{theorem}

\section{Locally private selection}\label{sec:ldp}

We start this section by defining the local privacy model. We then demonstrate that every SQO(C) and SQA can be implemented under the local privacy constraint, enabling us to construct a locally private version of Algorithm~\ref{alg:ours}. 
\begin{definition}[$\eps$-Local Randomizer, \citep{kasiviswanathan2011can}]
    We say a randomized function $\cM: \cX \rightarrow \cY$ is an $\eps$-Local Randomizer ($\eps$-LR) if for all $x,x'\in \cX$ and any measurable subset $Y\subseteq \cY$ it satisfies 
    \begin{equation*}
        \prob{\cM(x)\in Y} \leq e^{\eps}\prob{\cM(x')\in Y}.
    \end{equation*}    
\end{definition}

The following definition enables us to abstract the ``single round access'' to a database through an LR oracle.
\begin{definition}[$\eps$-LR Oracle]\label{def:LROracle}
    Let $D=\{x_1,\ldots,x_m\}\in \cX^m$ be a database of size $m$. An $\eps$-LR oracle for $D$, denoted by $\Psi_D$, operates as follows: initially, it sets the total number of ``questions'' to 0, i.e., $s=0$. Then in each call to the oracle:
    \begin{itemize}
        \item  $\Psi_D$ receives a series of $\eps$-LRs $(\cM_i)_{i=1}^n$ (for some $n>0$).
        \item If $n+s>m$, then the oracle outputs NULL.
        \item Otherwise, it outputs $\Psi_D\left((\cM_i)_{i=1}^n\right)=\left(\cM_i(x_{i+s})_{i=1}^{n}\right)$ and updates $s=s+n$.
    \end{itemize}  
\end{definition}

We stress that the above oracle $\Psi_D$ uses each data point only \textit{once}. As a result, each $\cM_i$ being an $\eps$-LR implies that $\Psi_D$ is also $\eps$-LR.  
The construction laid out in Definition~\ref{def:LROracle} provides a machinery for implementing oracles under the local privacy constraint. More specifically, one can construct a locally private version of an SQOC by composing a given local randomizer with the set indicator functions. To this goal, we choose the most well-known local randomizer, namely the randomized response \citep{warner1965randomized}, defined below. 
\begin{definition}[Randomized Response (RR)]
    The randomized response mechanism $R_\eps$ is a randomized function that receives $x\in \{0,1\}$ and outputs $x$ with probability $\frac{e^\eps}{e^\eps +1}$ and $1-x$ with probability $\frac{1}{e^\eps +1}$.
\end{definition}
It can be easily verified that $R_\eps$ is an $\eps$-LR. Moreover, according to the post-processing property of differential privacy (see, e.g., \cite[Proposition 2.1]{Dwork_Roth_book}), the composition of $R_\eps$ with any (measurable) function is still an $\eps$-LR. 
The following lemma expounds on how to privately implement an SQOC by post-processing of randomized response mechanisms.  

\begin{lemma}[Simulating an SQOC with an Unbiased RR]\label{lemma:sqo_simulation}
    For every probability space $(\cX,\cB,h)$ and every $\eps>0$, $\cO^{RR}_h$ defined as 
    \begin{equation*}
        \cO^{RR}_h((W_i)_{i=1}^k, \alpha, \beta, m) = \left(\frac{e^\eps +1}{e^\eps - 1}\left(\frac{1}{p}\sum_{j=p(i-1)+1}^{p.i} R_\eps\left(\indicator{{x_j}\in W_i}\right) - \frac{1}{e^\eps+1}\right)\right)_{i=1}^k,
    \end{equation*}
is a valid SQOC for $h$,  where $W_i\in \cB$, $p=O\left(\frac{\log m/\beta}{\alpha^2\min \{\eps^2,1\}}\right)$ and $D=\{x_1, \ldots, x_{pk}\}$ is fresh i.i.d.\ samples generated from $h$. 
\end{lemma}
Note that by setting the number of critical queries to the total number of queries the above lemma can be also used to simulate any SQO.

Next, we seek to show how to privately implement our hypothesis selection algorithm. To do so, we first need to formally define locally private algorithms. Informally speaking, an $\eps$-LDP algorithm is a multi-round algorithm with access to an $\eps$-LR oracle. In each round, the algorithm observes the outputs of the previous round and defines a series of $\eps$-LRs. It then calls the oracle and computes a function of the returned values of the oracle as its output for that round. 
\begin{definition}[$\eps$-LDP Algorithm]
    An $\eps$-LDP algorithm with $t$ rounds is a randomized function $\cA_\eps$ that has access to an $\eps$-LR oracle $\Psi_D$ (for some database $D$), and outputs $y^t$ using the following recursive procedure: at each round $j\in [t]$, the algorithm defines a function $f^j$ depending on the outcomes of the previous round (i.e., $y^{j-1}$), picks a series of $n_j$ $\eps$-LRs $(\cM_i^j)_{i=1}^{n_j}$, and computes         
        $y^j = f^j\left(\Psi_D\left((\cM^j_i)_{i=1}^{n_j}\right)\right)$. 
\end{definition}
Intuitively, one can employ a strategy akin to that in Lemma~\ref{lemma:sqo_simulation} to simulate SQA by an $\eps$-LDP algorithm using $\cO^{RR}_h$, thus restricting the access to the database only through an $\eps$-LR oracle. 
\begin{observation}\label{obs}
    Any SQA with $t$ rounds can be simulated with an $\eps$-LDP algorithm $\cA_{\eps}$ with $t$ rounds. Moreover, the $\eps$-LR oracle associated with $\cA_{\eps}$ requires $O\left(\frac{n_i\log m_i/\beta_i}{\alpha_i^2\min \{\eps^2,1\}}\right)$ data points in round $i\in [t]$, where $n_i$ is the size of the workload of queries, $\alpha_i,\beta_i$ are the accuracy and failure parameters, respectively, and $m_i\in[n_i]$ is the number of critical queries.   
\end{observation}

This observation can be proved by a direct application of Lemma~\ref{lemma:sqo_simulation} and provides a systemic framework for constructing an $\eps$-LDP version of  Algorithm~\ref{alg:ours}. The following theorem, whose proof is given in Appendix~\ref{app:ours}, presents the guarantees of such a locally private hypothesis selector. 

\begin{theorem}[$\eps$-LDP Implementation of Algorithm~\ref{alg:ours}]\label{thm:ours}
   Let $\cF=\{f_1,\ldots,f_k\}$ be a set of $k$ distributions, $\eps>0$ and $\alpha,\beta \in (0,1)$.  
   Let $\cO^{RR}_{h}$, as defined in Lemma~\ref{lemma:sqo_simulation}, be the oracle for the (unknown) distribution $h$. 
   Algorithm $\textsc{BOKSERR}(\cF,\cO^{RR}_{h},\alpha,\beta)$ is an $\eps$-LDP algorithm that requires a database of $O\Big(\frac{k (\log \frac{1}{\beta})^2}{\alpha^2\min \{\eps^2,1\}}\Big)$ i.i.d.\ samples from $h$, runs in $\left(6+4\log\log\frac{3}{\beta}\right)\log\log k$ rounds, and returns a distribution $\hat{f}\in\cF$ such that with probability at least $1-\beta$ we have $d_{TV}(h,\hat{f}) \leq 9d_{TV}(h,\cF)+\alpha$. 
\end{theorem}

{
  \printbibliography
}
\newpage
\appendix
\appendix
\section{Miscellaneous facts}

\begin{lemma}[Hoeffding's Inequality]\label{lemma:hoeffding}
    Let $X_1,\ldots,X_n$ be i.i.d. random variables with $\mu=\expect{X_i}$ and $\prob{a\leq X_i \leq b}=1$ for all $i\in[n]$. Then for every $\alpha>0$ we have
    \begin{equation*}
        \prob{\left|\frac{1}{n}\sum_{i=1}^n X_i-\mu \right| > \alpha} \leq \exp\left(-\frac{2n\alpha^2}{(b-a)^2}\right).
    \end{equation*}

\end{lemma}
\section{Analysis of Theorem~\ref{thm:knockout} (Boosted knockout)}\label{app:knockout}
We prove each part of Theorem~\ref{thm:knockout} in one of the following sections.

\begin{algorithm}
\SetFuncSty{textsc}
  \SetKwFunction{knockout}{Boosted-Knockout}
  \SetKwProg{Fn}{procedure}{:}{end procedure}
  \KwIn{A set $\cF=\{f_1,\ldots,f_k\}$ of $k$ distributions, Oracle $\cO_h$, parameters $\alpha,\beta>0, t\geq 1$}
  \KwOut{Two lists of candidate distributions from $\cF$}
    \Fn{\knockout{$\cF,\cO_h,\alpha,\beta,t$}}{
    Sample $n = 8\log\frac{1}{\beta}.2^{t.\log{\frac{3}{2}}}$ distributions randomly from $\cF$ and copy them to $\cK_2$\;\\
    $\cF_1\gets \cF$\;\\
    \For{$i \in [t]$}{
   $\forall f\in \cF_i$, let $w[f] \gets 0$ \Comment{To record the number of wins of $f$}\;\\
    $\cH \gets \emptyset$, $\cG \gets \emptyset$, $r\gets \lceil 32(\frac{4}{3})^i\log{\frac{1}{\beta}}\rceil$\;\\
    \For{$j \in [r]$}{ 
    Randomly pair the distribution in $\cF_i$, copy the pairs to $\cG$ \Comment{$\frac{|\cF_i|}{2}\leq |\cG| \leq \frac{|\cF_i|}{2}r$}\;
    }
    \For{every pair $p=(f,f')\in \cG$}{  $B_p \gets Sch(f,f')$\; 
    }
    $W = (B_p)_{p\in\cG}$, $(y_p)_{p\in\cG} = \cO_h(W,\alpha,\beta/ 2^i,r)$\;\\
    \For{every pair $p=(f,f')\in \cG$}{
    \uIf{$\textsc{SCHEFF\'E}(f,f',y_p)$ returns $f$}{ $w[f] \gets w[f]+1$\;}   
    \Else{$w[f'] \gets w[f']+1$\;}
    }
    \For{every $f \in \cF_i$}{
    \uIf{$w[f] \geq \frac{3}{4}r$}{$\cH \gets \cH \cup f$\;
    }  
    }
     $\cF_{i+1} \gets \cH$\;
    }
    $\cK_1 \gets \cF_{t+1}$\;\\
    \Return $\cK_1,\cK_2$\;
    }
    \caption{Boosted Knockout}\label{alg:ko}
\end{algorithm}
\subsection{Size of the returned lists}
We know that at the end of each round the set of distributions $\cF$ is updated and only the distributions that are in $\cH$ in that round will proceed to the next round. In the following lemma, we show that the number of remaining distributions decays exponentially with the number of rounds.
\begin{lemma}\label{lemma:ko_size}
    Let $k_0=k$ and denote by $k_i$ be the number of distributions that are remaining from the initial set of distributions $\cF=\{f_1,\ldots,f_k\}$ at the end of round $i\in [t]$ of Algorithm~\ref{alg:ko}, i.e., the distributions in $\cH$ at the end of round $i$. Then we have $|k_i|\leq \frac{k}{(\frac{3}{2})^i}$.
\end{lemma}
\begin{proof}
    We know that at round $i$ we pair the remaining distributions randomly for $r$ times and each time we run Scheff\'e tests between the pairs of distributions. Therefore, the total number of  Scheff\'e  tests run in round $i$ is equal to $\frac{k_i}{2}r$. We also know that only the distributions that are returned by at least $\frac{3}{4}r$ of the Scheff\'e  tests will be included in $\cH$ and remain for the next round. Therefore, each one of the $k_{i+1}$ distributions that will remain for the next round (i.e., round $i+1$) must have been involved in at least $\frac{3}{4}r$ tests. We can then write
    \begin{equation*}
        k_{i+1}\cdot\frac{3r}{4} \leq \frac{k_i}{2}r.
    \end{equation*}
    Thus, we can conclude that $k_{i+1}\leq\frac{k_i}{(\frac{3}{2})}$. The lemma can then be proved by using induction on $i\in[t]$.
\end{proof}
\begin{proof}
    The fact that $|\cK_1|\leq\frac{k}{2^{t.\log\frac{3}{2}}}$ can be verified easily from Lemma~\ref{lemma:ko_size} and $|\cK_2|\leq 8\log\frac{1}{\beta}2^{t.\log\frac{3}{2}}$ comes from the properties of the algorithm.
\end{proof}
\subsection{TV guarantees of the returned lists}
\begin{proof}
        Let $f^*$ be a distribution in $\cF$ such that $d_{TV}(f^*,h)=\min_{f_\in\cF}d_{TV}(h,f) = d_{TV}(h,\cF)$ and let $\gamma=\frac{\alpha}{d_{TV}(h,\cF)}$. We only use $\gamma$ for the analysis and we do not need to know its value-- in fact the values of $\gamma$ and $d_{TV}(h,\cF)$ are unknown to us. Denote by $S=\{f:d_{TV}(h,f)\leq (3+\gamma)d_{TV}(h,\cF), f\in \cF\}$ the set of distributions in $\cF$ that their TV distance to $h$ is within a $3+\gamma$ factor of the minimum TV distance $d_{TV}(h,\cF)$. We define $\zeta=\frac{\left|S\right|}{|\cF|}$ as the ratio of these distributions and base our analysis on two different ranges for the value of $\zeta$. We show that when $\zeta$ is smaller than $\frac{1}{8.\left(\frac{3}{2}\right)^t}$ then with high probability $f^*\in\cK_1$ and if it is larger than this value then, with high probability, $S\cap \cK_2 \neq \emptyset$. 

        {\bf Case 1:} $\zeta \leq\frac{1}{8.\left(\frac{3}{2}\right)^t}$.
        We want to prove that in this case $f^*$ will not be eliminated in any round $i\in[t]$ and makes it to the last round, i.e., $f^*\in\cK_1$. In each round of boosted knockout, the distributions are randomly paired for $r$ times and each time a Scheff\'e test is done between every pair, meaning that each distribution is involved in $r$ tests. We then select the distributions that are returned by Scheff\'e  tests for at least $\frac{3r}{4}$ times and move them to the next round, while eliminating the rest. Since $\alpha = \gamma.d_{TV}(h,\cF)$ we know from the guarantees of Scheff\'e  test that for any $f'\in\cF$, with probability at least $1-\beta$, the $\textsc{SCHEFF\'E}(f^*,f',\cO_h,\alpha,\beta)$ test may not return $f^*$ only if $d_{TV}(h,f')\leq 3.d_{TV}(h,f^*)+\alpha = (3+\gamma)d_{TV}(h,f^*)=(3+\gamma)d_{TV}(h,\cF)$, i.e., it may not return $f^*$ only if $f'$ is in $S$. For all $j\in[r]$, let $Z_{ij}$ be a Bernoulli random variable that is equal to $1$ if $f^*$ is paired with a distribution in $S$ at the $j$th time and is equal to $0$ otherwise. Consequently, the number of times that $f^*$ is tested against a distribution in $S$ at round $i$ is equal to $\sum_{j=1}^r Z_{ij}$. Moreover, let $\cF_i$ refer to the set of distributions that are remained at the beginning of round $i\in[t]$ and denote by $\zeta_i=\frac{|S|}{\cF_i}$ the ratio of the distributions in $S$ relative to the distributions that are remained at the beginning of round $i\in [t]$, e.g., $\cF_1=\cF$ and $\zeta _1=\zeta$. Since $\expect{Z_{ij}}=\zeta_i$ for all $i\in[t]$ and $j\in[r]$, we can write that
        \begin{equation}\label{eq:ko_proof_1}
        \begin{aligned}
            &\prob{\text{$f^*$ is eliminated at round $i$}} \\
            &=\prob{\text{$f^*$ is not returned by at least $\frac{r}{4}$ of the Scheff\'e tests in which it is tested at round $i$}}\\
            & \leq \prob{\sum_{j=1}^r Z_{ij}>\frac{r}{4}} = \prob{\frac{1}{r}\sum_{j=1}^r Z_{ij}>\frac{1}{4}} = \prob{\frac{1}{r}\sum_{j=1}^r Z_{ij}-\zeta_i>\frac{1}{4}-\zeta_i}\\&\leq \exp\left(-2r(\frac{1}{4}-\zeta_i)^2\right) 
             = \exp\left(-2r(\frac{1}{4}-\zeta_i)^2\right) = \exp\left(-64(\frac{4}{3})^i\log\frac{1}{\beta}(\frac{1}{4}-\zeta_i)^2\right)\\& = \beta^{64(\frac{4}{3})^i(\frac{1}{4}-\zeta_i)^2},       
            \end{aligned}
        \end{equation}
where the last inequality follows from Hoeffding's inequality (Lemma~\ref{lemma:hoeffding}). From Lemma~\ref{lemma:ko_size} we know that at the beginning of round $i$ we have $\cF_i = k_{i-1}\leq \frac{k}{(\frac{3}{2})^{i-1}}$. Therefore, we can conclude that $\zeta_i\leq (\frac{3}{2})^{i-1}\zeta$, i.e., the ratio of the distributions in $S$ relative to the distributions that are remained grows at most by a factor $3/2$ at each round. We also know that $\zeta \leq\frac{1}{8.\left(\frac{3}{2}\right)^t}$. Hence $\zeta_i\leq \frac{1}{8}(\frac{2}{3})^{t-i+1}\leq \frac{1}{12}$ for all $i\in[t]$ and 
\begin{equation*}
    \begin{aligned}
        64(\frac{1}{4}-\zeta_i)^2 \geq 64.\left(\frac{1}{4}-\frac{1}{12}\right)^2\geq \frac{16}{9}.
    \end{aligned}
\end{equation*}
From the above equation and Equation~\ref{eq:ko_proof_1} we can write that
  \begin{equation*}
        \begin{aligned}
            &\prob{f^*\notin \cK_1}  =\sum_{i=1}^t\prob{\text{$f^*$ is eliminated at round $i$}} \leq \sum_{i=1}^t \beta^{\frac{16}{9}(\frac{4}{3})^i}\leq \sum_{i=1}^\infty \beta^{\frac{16}{9}(\frac{4}{3})^i},
            \end{aligned}
        \end{equation*}
which is smaller than $\beta$ for $\beta<1/2$. Therefore, we have proved that if $\zeta \leq\frac{1}{8.\left(\frac{3}{2}\right)^t}$, then with probability at least $1-\beta$ (for $\beta<0.5$) we have $f^*\in \cK_1$ and thus $d_{TV}(h,\cK_1) = d_{TV}(h,\cF)$.

{\bf Case 2:} $\zeta >\frac{1}{8.\left(\frac{3}{2}\right)^t}$. In this case we prove that with probability at least $1-\beta$ a distribution ins $S$ will be returned in $\cK_2$. We can write
\begin{equation*}
    \begin{aligned}
        \prob{S\cap \cK_2 = \emptyset} \leq (1-\zeta)^{n} \leq e^{-\zeta n} = e^{-8\zeta \log\frac{1}{\beta}.2^{t.\log{\frac{3}{2}}}} = \left(e^{\log\beta}\right)^{8\zeta 2^{t.\log{\frac{3}{2}}}} = \beta^{8\zeta 2^{t.\log{\frac{3}{2}}}}.
    \end{aligned}
\end{equation*}
We know that $\zeta >\frac{1}{8.\left(\frac{3}{2}\right)^t}$ and , therefore,
\begin{equation*}
    8\zeta 2^{t.\log{\frac{3}{2}}} > \frac{1}{8.\left(\frac{3}{2}\right)^t}. 8.2^{t.\log{\frac{3}{2}}} = 1.
\end{equation*}
Combining the above two equations we can conclude that $\prob{S\cap \cK_2 = \emptyset} \leq \beta$.
In both cases, we require that, in round $i$, all the critical queries of that round are answered accurately with probability more than $1-\frac{\beta}{2^i}$. Therefore, taking a union bound implies that all the critical queries are estimated accurately with probability more than $1-\sum_i \frac{\beta}{2^i} \geq 1-\beta$. Taking another union bound implies that $f^*$ is not eliminated with probability more than $1-2\beta$.
\end{proof}
\subsection{Number of total and critical queries}
\begin{proof}
We know that a single Scheff\'e  test can be done by one statistical query to the oracle. It is easy to verify that at round $i\in[t]$ the total number of statistical queries is equal to $r.\frac{k_{i-1}}{2}$, which is exactly equal to the number of Scheff\'e tests run in that round. From Lemma~\ref{lemma:ko_size} we know that $k_{i-1}\leq \frac{k}{(\frac{3}{2})^{i-1}}$. Therefore the number of queries at round $i$ is less than or equal to $r.\frac{k}{2(\frac{3}{2})^{i-1}} = \frac{32(\frac{4}{3})^i\log{\frac{1}{\beta}}.k}{2(\frac{3}{2})^{i-1}}$. Therefore, the total number of queries is less than or equal to 
\begin{equation*}
\sum_{i=1}^t \frac{32(\frac{4}{3})^i k \log{\frac{1}{\beta}}}{2(\frac{3}{2})^{i-1}} = 16\frac{4}{3}\sum_{i=1}^t\frac{(\frac{4}{3})^{i-1} k \log{\frac{1}{\beta}}}{(\frac{3}{2})^{i-1}} =  \frac{64 k\log{\frac{1}{\beta}}}{3}\sum_{i=1}^\infty  (\frac{8}{9})^{i-1} \leq 192 k\log{\frac{1}{\beta}} = O\left(k\log \frac{1}{\beta}\right).
\end{equation*}
Now we derive an upper bound on the number of critical queries. Note that if $\zeta \leq \frac{1}{8.\left(\frac{3}{2}\right)^t}$ the guarantees of boosted knockout requires that with probability at least $1-\beta$ we have $f^*\in \cK_1$ and the way that we proved this fact relied on finding the probability that $f^*$ is returned by at least $\frac{3}{4}$ fraction of the $r$Scheff\'e  tests in which it is involved at each round. Also note that the output of all the other Scheff\'e  tests that do not have $f^*$ as their input is not important since we do not base our analysis on the exact value of $\zeta_i$. We rather simply bound it by its worst-case-- when none of the distributions in $S$ are eliminated even until the last round $t$ and $\zeta_i$ is multiplied by a factor of $3/2$ at each round. Therefore, the only queries that are critical and can change the output of the algorithm are those that are related to the Scheff\'e  tests in which $f^*$ is included. On the other hand if $\zeta \leq \frac{1}{8.\left(\frac{3}{2}\right)^t}$, we proved that with probability at least $1-\beta$ a distribution in $S$ is sampled and copied to $\cK_2$ and this does not require any queries at all. Consequently, we can say that the number of critical queries in round $i$ is less than or equal $r = 32(\frac{4}{3})^i\log\frac{1}{\beta}$, and thus the total number of critical queries is less than or equal to $96(\frac{4}{3})^{t_1+1}\log\frac{1}{\beta}$.
\end{proof}

\section{Analysis of Theorem~\ref{thm:srr} (Boosted sequential round-robin)}\label{app:SRR}

Before proving Theorem~\ref{thm:srr} we provide an intuition as to why the repetition of the round-robin tournament in each round of boosted SRR is necessary to obtain the logarithmic dependency of sample complexity on $1/\beta$ in Theorem~\ref{alg:ours}.

{\bf Analysis of $\beta$ in SRR}. As already discussed, we wish to make sure that if $\zeta$ is smaller than some threshold then $f^*$ will not be grouped with any distribution in $S$ with high probability. It is easy to verify that in the ``original'' SRR \citep{gopi2020locally}, where the distribution are grouped only once in each round, the failure can happen with probability $O(\zeta\sqrt{k})$. In order to make sure that this probability is smaller than $\beta$, we have to set the threshold to $\beta/\sqrt{k}$. On the other hand, we have to make sure that when $\zeta$ is larger than the threshold, then a sub-sample of the input distributions will include a distribution from $S$ with probability at least $1-\beta$. The probability of failure for this event will depend on the size of the sub-sample $\cR_2$ and is bounded by $O(\exp(-\zeta n))$. To make sure that this is smaller than $\beta$, we would have to set $n = O(\frac{\sqrt{k}\log(1/\beta)}{\beta})$. But notice that the sample complexity of MDE-variant is quadratic in the size of its input and, therefore, we would get a quadratic dependence on $1/\beta$ in the sample complexity. However, repeating the round-robins for several times in each round of boosted SRR makes sure that the probability that $f^*$ is not eliminated increases sufficiently such that we can set a smaller threshold for $\zeta$ and, thus, sub-sample smaller number of distributions.

Similar to Theorem~\ref{thm:knockout}, we prove each part of Theorem~\ref{thm:srr} in one of the following sections.

\begin{algorithm}
\SetAlgoLined
\SetFuncSty{textsc}
  \SetKwFunction{mrr}{Multi-Round-Robin}
  \SetKwFunction{srr}{Boosted-Sequential-Round-Robin}
  \SetKwProg{Fn}{procedure}{:}{end procedure}
  \KwIn{A set $\cF=\{f_1,\ldots,f_k\}$ of $k$ distributions, Oracle $\cO_h$, parameters $\alpha,\beta,\eta>0, t\geq 1$}
  \KwOut{Two lists of candidate distributions from $\cF$}
    \Fn{\mrr{$\cF,\cO_h,\alpha,\beta,\eta$}}{
    $\cH \gets \emptyset$, $\cG \gets \emptyset$\;\\
    Randomly partition the distributions in $\cF$ into $|\cF|/\eta$ sets of size $\eta$, copy the partitions to $\cG$\;\\
    \For{every partition $\cG_i\in \cG$}{
     $g_i \gets \textsc{Round-Robin}(\cG_i,\cO_h,\alpha,\beta)$\;\\
     $\cH \gets \cH \cup g_i$\;
    }
     \Return $\cH$\;
    }
    \Fn{\srr{$\cF,\cO_h,\alpha,\beta,\eta,t$}}{
    Sample $2\eta^{2^t}\log(\frac{1}{\beta})$ distributions randomly from $\cF$ and copy them to $\cR_2$\;\\
    $\cF_1 \gets \cF$\;\\
    \For{$i \in [t]$}{
     $\cH \gets \emptyset$\;\\
     \For{$j \in \lceil\log(\frac{1}{\beta})\rceil$}{
     $\cH_j \gets \textsc{Multi-Round-Robin}(\cF_i,\cO_h,\alpha,\beta /2^i,\eta)$\;\\
     $\cH \gets \cH \cup \cH_j$\;
    }
    $\cF_{i+1} \gets \cH$\;\\
    $\eta \gets \eta^2$\;\\
    }
    $\cR_1 \gets \cF_{t+1}$\;\\
 \Return $\cR_1,\cR_2$\;\\
    }
\caption{Boosted Sequential-Round-Robin}\label{alg:srr}
\end{algorithm}
\subsection{Size of the returned lists}
We first state the following fact regarding the size of the output class of distributions of a single call to multi-round-robin.
\begin{fact}\label{fact:mrr_size}
    Let $\cF=\{f_1,\ldots,f_k\}$ be a set of $k$ distributions and let $\alpha,\beta>0$. 
   Let $\cO_h$ be a SQOC of accuracy $\alpha$ for the (unknown) distribution $h$.
   Algorithm $\textsc{Multi-Round-Robin}(\cF,\cO_h,\alpha,\beta,\eta)$ returns a set of distributions $\cH$ with $|\cH| = k/\eta$.
\end{fact}
\begin{proof}
 The fact is an immediate consequence of the facts that the number of partitions is $|\cG|=k/\eta$ and that round-robin outputs a single distribution.
\end{proof}
We know that at each round of the boosted sequential-round-robin the set of distributions $\cF$ is updated and only the distributions in $\cH$ will remain for the next round. We first state the following fact and in Lemma~\ref{lemma:srr_size} we state the total number of remaining distributions at each round of the boosted sequential-round-robin.
\begin{fact}\label{fact:srr_eta}
    Let $\eta_1=\eta$ and $\eta_i$ denote the updated value of $\eta$ at the beginning of round $i\in[t]$ of Algorithm~\ref{alg:srr}. We have $\eta_i = \eta^{2^{i-1}}$. 
\end{fact}

\begin{lemma}\label{lemma:srr_size}
     Let $k_0=k$ and $k_i$ denote the number of distributions that are remaining from the initial set of distribtuions $\cF=\{f_1,\ldots,f_k\}$ at the end of round $i\in[t]$ of Algorithm~\ref{alg:srr}, i.e., the distributions in $\cH$ at the end of round $i$. Then we have $k_i \leq \frac{k(\log\frac{1}{\beta})^i}{\eta^{2^{i}-1}}$.
\end{lemma}
\begin{proof}
    In round $i\in[t]$ of the boosted sequential-round-robin, the remaining set of distributions are used as an input to multi-round-robin for $\log\frac{1}{\beta}$ times. From Facts~\ref{fact:mrr_size} and \ref{fact:srr_eta} we know that for each one of the $\log\frac{1}{\beta}$ runs of multi-round-robin at round $i$, the size of the output set of distributions is $\frac{k_{i-1}}{\eta_i} = \frac{k_{i-1}}{\eta^{2^{i-1}}}$. Therefore, the total number of distributions that are remained at the end of round $i$ is at most $\frac{k_{i-1}.\log\frac{1}{\beta}}{\eta^{2^{i-1}}}$. We can then conclude the lemma using an induction on $i\in[t]$. It is easy to verify that at the end of round $1$ we have $k_1 = \frac{k\log\frac{1}{\beta}}{\eta}$. Assume we have $k_i \leq \frac{k(\log\frac{1}{\beta})^i}{\eta^{2^{i}-1}}$. Then from the above argument we can conclude that $k_{i+1} \leq \frac{k_{i}\log\frac{1}{\beta}}{\eta^{2^{i}}} = \frac{k(\log\frac{1}{\beta})^{i+1}}{\eta^{2^{i}}\eta^{2^{i}-1}} = \frac{k(\log\frac{1}{\beta})^{i+1}}{\eta^{2^{i+1}-1}}$.
\end{proof}

\begin{proof}
    From Lemma~\ref{lemma:srr_size} we know that size of the returned list $\cR_1$ is equal to $|\cR_1| \leq \frac{k(\log\frac{1}{\beta})^{t}}{\eta^{2^{t}-1}}$. The fact that $|\cR_2| = 2\eta^{2^t}\log\frac{1}{\beta}$ simply follows from the definition of algorithm.
\end{proof}
\subsection{TV guarantees of the returned lists}
\begin{proof}
        Let $f^*$ be a distribution in $\cF$ such that $d_{TV}(f^*,h)=\min_{f_\in\cF}d_{TV}(h,f) = d_{TV}(h,\cF)$ and let $\gamma=\frac{\alpha}{d_{TV}(h,\cF)}$. We only use $\gamma$ for the analysis and we do not need to know its value-- in fact the values of $\gamma$ and $d_{TV}(h,\cF)$ are unknown to us. Denote by $S=\{f:d_{TV}(h,f)\leq (3+\gamma)d_{TV}(h,\cF), f\in \cF\}$ the set of distributions in $\cF$ that their TV distance to $h$ is within a $3+\gamma$ factor of the minimum TV distance $d_{TV}(h,\cF)$. We define $\zeta=\frac{\left|S\right|}{|\cF|}$ as the ratio of these distributions and base our analysis on two different ranges for the value of $\zeta$. We show that when $\zeta$ is smaller than $\frac{1}{2\eta^{2^t}}$ then with high probability $f^*\in\cR_1$ and if it is larger than this value then, with high probability, $S\cap \cR_2 \neq \emptyset$. 

        {\bf Case 1:} $\zeta \leq\frac{1}{2\eta^{2^t}}$.
        We want to prove that in this case $f^*$ will not be eliminated in any round $i\in[t]$ and makes it to the last round, i.e., $f^*\in\cR_1$. In each round of the boosted sequential-round-robin, we run the multi-round-robin sub-routine for a total of $\log\frac{1}{\beta}$. For each run of multi-round-robin at round $i\in[t]$ the distributions are partitioned into $k_{i-1}/\eta_i$ groups and a round-robin is run on the distributions in each group. The outputs of all round-robins will then proceed to the next round.  Since $\alpha = \gamma.d_{TV}(h,\cF)$ we know from the guarantees of Scheff\'e test that for any $f'\in\cF$, with probability at least $1-\beta$, the $\textsc{SCHEFF\'E}(f^*,f',\cO_h,\alpha,\beta)$ test may not return $f^*$ only if $d_{TV}(h,f')\leq 3.d_{TV}(h,f^*)+\alpha = (3+\gamma)d_{TV}(h,f^*)=(3+\gamma)d_{TV}(h,\cF)$, i.e., it may not return $f^*$ only if $f'$ is in $S$. In the following, we will bound the probability that $f^*$ is eliminated by the probability, for all $\log\frac{1}{\beta}$ repetitions of multi-round-robin, the partition that contains $f^*$ will also contain at least one distribution from $S$ . Otherwise, $f^*$ will be returned by every Scheff\'e test in that partition and thus by the multi-round-robin and proceeds to the next round.
        
        Let $\cF_i$ refer to the set of distributions that are remained at the beginning of round $i\in[t]$ and denote by $\zeta_i=\frac{|S|}{\cF_i}$ the ratio of the distributions in $S$ relative to the distributions that are remained at the beginning of round $i\in [t]$, e.g., $\cF_1=\cF$ and $\zeta _1=\zeta$. We can write that
        \begin{equation}\label{eq:srr_proof_1}
        \begin{aligned}
            &\prob{\text{$f^*$ is eliminated at round $i$}} \\
            &\prob{\text{$f^*$ is grouped with at least one distribution in $S$ for all $\log \frac{1}{\beta}$ repetitions at round $i$}} \\
            & \leq (\zeta_i\eta_i)^{\log\frac{1}{\beta}},      
            \end{aligned}
        \end{equation}
where the last inequality follows from the fact that the ratio of distributions in $S$ and the size of the groups at round $i\in[t]$ are $\zeta_i$ and $\eta_i$, respectively. 
We also know that $\zeta_i\leq \eta^{2^{i-1}-1}\zeta$. To see this, note that in each of the $\log\frac{1}{\beta}$ runs of multi-round-robin the same set $\cF_i$ of distributions are used as input and in the worst-case every distribution in $S$ will be among the $\cF_i/\eta_i$ distributions that are returned by each multi-round-robin. Similar to Lemma~\ref{lemma:srr_size}, we can prove the above bound on $\zeta_i$ using an induction on $i$.

Taking Fact~\ref{fact:srr_eta} and $\zeta \leq\frac{1}{2\eta^{2^t}}$ into account, we can write that for all $i \in [t]$,
\begin{equation*}
    \begin{aligned}
   (\zeta_i\eta_i)^{\log\frac{1}{\beta}} \leq  \left( \frac{\eta^{2^{i-1}-1}}{2\eta^{2^t}}\eta^{2^{i-1}}\right)^{\log\frac{1}{\beta}} = \left( \frac{\eta^{2^{i}-1}}{2\eta^{2^t}}\right)^{\log\frac{1}{\beta}}
        \end{aligned}
\end{equation*}
From the above equation and Equation~\ref{eq:srr_proof_1} we can write that
  \begin{equation*}
        \begin{aligned}
            &\prob{f^*\notin \cR_1}  =\sum_{i=1}^t\prob{\text{$f^*$ is eliminated at round $i$}} \\&\leq \sum_{i=1}^t  \left( \frac{\eta^{2^{i}-1}}{2\eta^{2^t}}\right)^{\log\frac{1}{\beta}} = \left(\frac{1}{2\eta^{2^t}} \right) ^{\log \frac{1}{\beta}}\sum_{i=1}^t  \left( \eta^{2^{i}-1}\right)^{\log\frac{1}{\beta}} = \beta ^{\log \left(2\eta^{2^t}\right)}\sum_{i=1}^t  
            \beta^{\log\frac{1}{\eta^{2^{i}-1}}} \\
            &= \beta ^{\log \left(2\eta^{2^t}\right)}\sum_{i=1}^t  
            \beta^{(2^{i}-1)\log\frac{1}{\eta}} = \beta ^{\log \left(2\eta^{2^t}\right)}\frac{\left(\beta^{\log\frac{1}{\eta}}\right)^{2^t} - 1}{\left(\beta^{\log\frac{1}{\eta}}\right) - 1} \\
            &\leq \beta ^{\log \left(2\eta^{2^t}\right)}.\left(\beta^{\log\frac{1}{\eta}}\right)^{2^t} 
             = \beta^{\log 2\eta^{2^t}\left(\frac{1}{\eta}\right)^{2^t}} = \beta.
            \end{aligned}
        \end{equation*}
Therefore, we have proved that if $\zeta \leq\frac{1}{2\eta^{2^t}}$, then with probability at least $1-\beta$ we have $f^*\in \cR_1$ and thus $d_{TV}(h,\cR_1) = d_{TV}(h,\cF)$.

{\bf Case 2:} $\zeta >\frac{1}{2\eta^{2^t}}$. In this case we prove that with probability at least $1-\beta$ a distribution in $S$ will be returned in $\cR_2$. We can write
\begin{equation*}
    \begin{aligned}
        \prob{S\cap \cR_2 = \emptyset} \leq (1-\zeta)^{n} \leq e^{-\zeta n} = e^{-\zeta 2\eta^{2^t}\log\frac{1}{\beta}} = \left(e^{\log\beta}\right)^{\zeta 2\eta^{2^t}} \leq \beta.
    \end{aligned}
\end{equation*}
In both cases, we require that, in round $i$, all the critical queries in that round are answered accurately with probability more than $1-\frac{\beta}{2^i}$. Therefore, taking a union bound implies that all the critical queries are estimated accurately with probability more than $1-\sum_i \frac{\beta}{2^i} \geq 1-\beta$. Taking another union bound implies that $f^*$ is not eliminated with probability more than $1-2\beta$.
\end{proof}
\subsection{Number of total and critical queries}
\begin{proof}
We know from Theorem~\ref{thm:rr} that a round-robin run on a set of $k$ distributions requires answer to $\frac{k(k-1)}{2}$ queries, all of which are critical. Thus, the number of total and critical queries are equal in boosted sequential-round-robin. Particularly, in round $i$, we run multi-round-robin for $\log \frac{1}{\beta}$ times. Each multi-round-robin partitions the set of distributions in round $i$ into $k/\eta_i$ groups of size $\eta_i$ and runs a round-robin on each group. Therefore, the total number of queries is less than or equal to
\begin{equation*}
    \begin{aligned}
        & \sum_{i=1}^t \frac{k_{i-1}}{\eta_i}\eta_i^2\log\frac{1}{\beta} =  \sum_{i=1}^t \frac{k\left(\log\frac{1}{\beta}\right)^{i-1}}{\eta^{2^{i-1}-1}}\eta_i\log\frac{1}{\beta} = \sum_{i=1}^t \frac{k\left(\log\frac{1}{\beta}\right)^{i-1}}{\eta^{2^{i-1}-1}}\eta^{2^{i-1}}\log\frac{1}{\beta}  = \sum_{i=1}^t  k\eta \left(\log\frac{1}{\beta}\right)^{i}\\
        & \leq k\eta \left(\log\frac{1}{\beta}\right) ^ {t+1},
    \end{aligned}
\end{equation*}
where the last inequality holds if $\beta \leq 1/4$.
\end{proof}

\section{Analysis of Theorem~\ref{thm:ksrm} (BOKSERR)}\label{app:ours}
We first prove the TV guarantees of the returned distribution and then find the number of total and critical queries.
\subsection{TV guarantees of the returned distribution}
\begin{proof}
        Let $f^*$ be a distribution in $\cF$ such that $d_{TV}(f^*,h)=\min_{f_\in\cF}d_{TV}(h,f) = d_{TV}(h,\cF)$ and let $\gamma=\frac{\alpha}{6d_{TV}(h,\cF)}$. We only use $\gamma$ for the analysis and we do not need to know its value-- in fact the values of $\gamma$ and $d_{TV}(h,\cF)$ are unknown to us. Denote by $S=\{f:d_{TV}(h,f)\leq (3+\gamma)d_{TV}(h,\cF), f\in \cF\}$ the set of distributions in $\cF$ that their TV distance to $h$ is within a $3+\gamma$ factor of the minimum TV distance $d_{TV}(h,\cF)$. From Theorem~\ref{thm:knockout} we know that with probability at least $1-\frac{\beta}{3}$ either $f^* \in \cK_1$ or a distribution from $S$ is in $\cK_2$. Assume that $f^*\in\cK_1$. Then it is easy to verify that for $S'=\left\{f:d_{TV}(h,f)\leq (3+\gamma)d_{TV}(h,\cK_1), f\in \cK_1\right\}$ we have $S'\subseteq S$. Theorem~\ref{thm:srr} suggests that after running $\textsc{Boosted-Sequential-Round-Robin}(\cK_1,\cO_h,\alpha/3,\beta/3,\eta,t')$, given that $f^*\in\cK_1$, with probability at least $1-\frac{\beta}{3}$ either $f^{*}\in\cR_1$ or a distribution in $S'$ (and thus in $S$) is in $\cR_2$. Let $\Tilde{\cF} = \cR_1 \cup \cR_2 \cup \cK_2$. By a union bound argument, we can conclude that with probability at least $1-\frac{2\beta}{3}$ we have $\Tilde{\cF}\cap S \neq \emptyset$, i.e., $d_{TV}(h,\Tilde{\cF})\leq (3+\gamma)d_{TV}(h,\cF)$. Theorem~\ref{thm:mde} together with a union bound argument suggests that for $\hat{f}\gets \textsc{MDE-Variant}(\Tilde{\cF},\cO_h,\alpha/2,\beta/3)$ with probability at least $1-\beta$ we have 
        \begin{equation*}
        \begin{aligned}
            &d_{TV}(h,\hat{f}) \leq 3d_{TV}(h,\Tilde{\cF}) + \frac{\alpha}{2} \leq 3(3+\gamma)d_{TV}(h,\cF) + \frac{\alpha}{2} = 9d_{TV}(h,\cF) + 3\gamma. d_{TV}(h,\cF) + \frac{\alpha}{2} \\
             &=  9d_{TV}(h,\cF) +\frac{\alpha}{2} + \frac{\alpha}{2} = 9d_{TV}(h,\cF) + \alpha.
        \end{aligned}
        \end{equation*}
\end{proof}

\subsection{Number of rounds}
\begin{proof}
    It is easy to prove that the total number of rounds of Algorithm~\ref{alg:ours} is equal to the sum of the number of rounds required for Algorithms~\ref{alg:ko} and \ref{alg:srr} plus one additional round for MDE-variant. Therefore, we can conclude that the total number of rounds is equal to $t+t'+1 = \left(6+4\log\log\frac{3}{\beta}\right)\log\log k$.
\end{proof}

\subsection{Number of total and critical queries}
\begin{proof}
The number of total and critical queries are equal to the sum of the queries in boosted knockout, boosted sequential-round-robin, and MDE-variant sub-routines. Let $\Tilde{\cF} = \cR_1 \cup \cR_2 \cup \cK_2$. From Theorems~\ref{thm:knockout}, \ref{thm:srr}, and \ref{thm:mde}, we can conclude that the total number of queries is less than or equal to 
\begin{equation*}
\begin{aligned}
   & 192 |\cF|\log\frac{3}{\beta} +|\cK_1|\eta\left(\log\frac{3}{\beta}\right)^{t'+1} + \frac{|\Tilde{\cF}|(|\Tilde{\cF}|-1)}{2} \\
   & \leq 192 k \log\frac{3
   }{\beta} + \frac{k\cdot k'^{\frac{1}{2^{(t'+1)}}}\cdot (\log\frac{3}{\beta})^{t'+1}}{2^{t\log\frac{3}{2}}} + \left(\frac{k(\log\frac{3}{\beta})^{t_2}}{2^{t\log\frac{3}{2}}k'^{\frac{2^{t'}-1}{2^{(t'+1)}}}} + 8\log\frac{3}{\beta}2^{t \cdot \log\frac{3}{2}}+2k'^{\frac{2^{t'}}{2^{(t'+1)}}} \log \frac{3}{\beta}\right)^2 \\
   & \leq192k\log\frac{3}{\beta} + O\left(\frac{k(\log\frac{1}{\beta})^2}{(\log k)^{\log\log \frac{1}{\beta}}}\right) 
\end{aligned}
   \end{equation*}
where the last line follows by plugging the values of $k'$ and $t'$ based on Algorithm~\ref{alg:ours}. The number of critical queries is less than equal to
   \begin{equation*}
       \begin{aligned}
           96(\frac{4}{3})^{t+1}\log\frac{1}{\beta}+|\cK_1|\eta\left(\log\frac{3}{\beta}\right)^{t'+1} + \frac{|\Tilde{\cF}|(|\Tilde{\cF}|-1)}{2}  \leq O\left((\log k)^{\log\log\frac{1}{\beta}}\right) + O\left(\frac{k(\log\frac{1}{\beta})^2}{(\log k)^{\log\log \frac{1}{\beta}}}\right).
       \end{aligned}
   \end{equation*}
\end{proof}
\begin{algorithm}
  \SetAlgoLined
  \SetFuncSty{textsc}
  \SetKwFunction{ksrm}{BOKSERR}
  \SetKwProg{Fn}{procedure}{:}{end procedure}
  \KwIn{ A set $\cF=\{f_1,\ldots,f_k\}$ of $k$ distributions, Oracle $\cO_h$, parameters $\alpha,\beta>0$}\;
  \KwOut{$\hat{f}\in \cF$ such that $d_{TV}(h,\hat{f}) \leq 9d_{TV}(h,\cF)+\alpha$ with probability at least $1-\beta$}\;
     $t \gets \left(5+4\log\log \frac{3}{\beta}\right)\log\log k$, $t' \gets \log\log k -1$\;\\
     $k' \gets \frac{k}{2^{t\log\frac{3}{2}}}$, $\eta \gets k'^{\frac{1}{2^{(t'+1)}}}$ \;\\
    \Fn{\ksrm{$\cF,\cO_h,\alpha,\beta$}}{\;
$\cK_1,\cK_2,\cR_1,\cR_2 \gets \emptyset$\;\\
 $\cK_1,\cK_2 \gets \textsc{Boosted-Knockout}(\cF,\cO_h,\alpha/6,\beta/ 6,t)$\; \\
    $\cR_1,\cR_2 \gets \textsc{Boosted-Sequential-Round-Robin}(\cK_1,\cO_h,\alpha/6,\beta/6,\eta,t')$\;\\
    \Return $\textsc{MDE-Variant}(\cR_1 \cup \cR_2 \cup \cK_2,\cO_h,\alpha/2,\beta/3)$\;
    }\caption{\textsc{BOKSERR}}\label{alg:ours}
\end{algorithm}
\section{Missing proofs from Section~\ref{sec:ldp}}\label{app:ldp}
\subsection{Proof of Theorem~\ref{thm:ours}}
\begin{proof}
    We find the sample complexity of the algorithm by finding the samples needed to run the three sub-routines MDE-variant, which we denote by $m_1$, $m_2$, and $m_3$, respectively. 
    From Theorem~\ref{thm:knockout}, we know in round $i\in[t]$ of boosted knockout the algorithm asks a workload of $\frac{32(\frac{4}{3})^i k \log{\frac{1}{\beta}}}{2(\frac{3}{2})^{i-1}}$ queries and at most $32(\frac{4}{3})^i\log\frac{1}{\beta}$ of them are critical. Therefore, from Lemma~\ref{lemma:sqo_simulation}, the total number of samples that the oracle requires for the run of boosted knockout is equal to 
\begin{equation}\label{eq1:thm_local_alg}
\begin{aligned}
     m_1 &= \sum_{i=1}^{t} \frac{1}{\alpha^2\min\{\eps^2,1\}}\frac{32(\frac{4}{3})^i k \log{\frac{3}{\beta}}}{2(\frac{3}{2})^{i-1}} \cdot \log\left(\frac{2^i\cdot2\cdot 32(\frac{4}{3})^i\log\frac{3}{\beta}}{\beta}\right)\\
    &= \frac{64 k}{3\alpha^2\min\{\eps^2,1\}}\left(\log\frac{3}{\beta}\right)\sum_{i=1}^{t} \frac{(\frac{4}{3})^{i-1} \log\left(\frac{2^i\cdot 64(\frac{4}{3})^i\log\frac{3}{\beta}}{\beta}\right)}{(\frac{3}{2})^{i-1}}\\
    &=\frac{64 k\log{\frac{3}{\beta}}}{3\alpha^2\min\{\eps^2,1\}}\left (\log(\frac{1}{\beta}) \sum_{i=1}^\infty (\frac{8}{9})^{i-1}  + \sum_{i=1}^\infty i(\frac{8}{9})^{i-1}  + \sum_{i=1}^\infty  (\frac{8}{9})^{i-1} \log\left(\frac{256}{3} (\frac{4}{3})^{i-1}\log \frac{3}{\beta}\right)\right)\\
    &= O\left(\frac{k (\log \frac{1}{\beta})^2}{\alpha^2\min\{\eps^2,1\}}\right).
    \end{aligned}
\end{equation}

Theorem~\ref{thm:srr} suggests that in round $i\in[t']$ of boosted sequential-round-robin there are at most $|\cK_1|\eta\left(\log \frac{3}{\beta}\right)^i$ queries and all of them are critical. Therefore, we can find the samples required to run this sub-routine by writing
\begin{equation*}
    \begin{aligned}
     & m_2 = \sum_{i=1}^{t'} \frac{1}{\alpha^2\min\{\eps^2,1\}}.\frac{k}{2^{t\log\frac{3}{2}}}. k'^{\frac{1}{2^{(t'+1)}}}\left(\log\frac{3}{\beta}\right)^{i} . \log \left(   \frac{2\cdot 2^i\cdot\frac{k}{2^{t\log\frac{3}{2}}}. k'^{\frac{1}{2^{(t'+1)}}}\left(\log\frac{3}{\beta}\right)^{i}}{\beta}  \right)\\
     & = \frac{1}{\alpha^2\min\{\eps^2,1\}} \frac{2 k}{(\log k)^{\log\frac{3}{2}(5+4\log\log \frac{3}{\beta})(1+1/\log k)}} \cdot \\
     &\sum_{i=1}^{t'} \left(\log\frac{3}{\beta}\right)^{i} \cdot\log \left(  \frac{2^i\cdot 4k}{(\log k)^{\log\frac{3}{2}(5+4\log\log \frac{3}{\beta})(1+1/\log k)}\cdot\beta}\cdot \left(\log\frac{3}{\beta}\right)^i \right)
    \end{aligned}
\end{equation*}
We first upper bound the term with the sum.
\begin{equation*}
    \begin{aligned}
     &\sum_{i=1}^{t'} \left(\log\frac{3}{\beta}\right)^{i} \cdot\log \left(  \frac{2^i\cdot 4k}{(\log k)^{\log\frac{3}{2}(5+4\log\log \frac{3}{\beta})(1+1/\log k)}\cdot\beta}\cdot \left(\log\frac{3}{\beta}\right)^i \right)\\
    &= \log \left(\frac{4k}{(\log k)^{\log\frac{3}{2}(5+4\log\log \frac{3}{\beta})(1+1/\log k)}\cdot\beta}\right)\sum_{i=1}^{t'} \left( \log \frac{3}{\beta}\right)^i
    +\sum_{i=1}^{t'} \log\left(\frac{3}{\beta}\right)^i \log\left(2^i\left (\log \frac{3}{\beta}\right)^i\right) \Biggl)\\
    & \leq \log \left( \frac{4k}{(\log k)^{\log\frac{3}{2}(5+4\log\log \frac{3}{\beta})(1+1/\log k)}\cdot \beta}\right) \frac{\left(\log\frac{3}{\beta}\right)^{t'+1} - 1}{\left(\log\frac{3}{\beta}\right) - 1} + \log\log k \cdot\log(2\log \frac{3}{\beta})\cdot\sum_{i=1}^{t'}\left(\log\frac{3}{\beta}\right)^i\\
     &  \leq \log \left( \frac{4k}{(\log k)^{\log\frac{3}{2}(5+4\log\log \frac{3}{\beta})(1+1/\log k)}\cdot\beta}\right) \left(\log\frac{3}{\beta}\right)^{t'+1} + \log\log k \cdot\log(2\log \frac{3}{\beta})\left(\log\frac{3}{\beta}\right)^{t'+1}\\
     &= \log \left( \frac{4k}{(\log k)^{\log\frac{3}{2}(5+4\log\log \frac{3}{\beta})(1+1/\log k)}.\beta}\right) \left(\log k\right)^{\log\log\frac{3}{\beta}}+ \log\log k \cdot \log(2\log \frac{3}{\beta}) \left(\log k\right)^{\log\log\frac{3}{\beta}},
    \end{aligned}
\end{equation*}
where in the last line we used the fact that $t'=\log\log k -1$. We can therefore upper bound $m_2$ as follows.
\begin{equation}\label{eq2:thm_local_alg}
    \begin{aligned}
     & m_2=\frac{1}{\alpha^2\min\{\eps^2,1\}} \frac{2 k}{(\log k)^{\log\frac{3}{2}(5+4\log\log \frac{3}{\beta})(1+1/\log k)}} \cdot \\
     &\sum_{i=1}^{t'} \left(\log\frac{3}{\beta}\right)^{i} \cdot\log \left(  \frac{2^i\cdot 4k}{(\log k)^{\log\frac{3}{2}(5+4\log\log \frac{3}{\beta})(1+1/\log k)}\cdot\beta}\cdot \left(\log\frac{3}{\beta}\right)^i \right)\\
     &\leq \frac{1}{\alpha^2\min\{\eps^2,1\}} \cdot \frac{2 k \left(\log k\right)^{\log\log\frac{3}{\beta}}}{(\log k)^{\log\frac{3}{2}(5+4\log\log \frac{3}{\beta})(1+1/\log k)}} \log \left( \frac{4k}{(\log k)^{\log\frac{3}{2}(5+4\log\log \frac{3}{\beta})(1+1/\log k)}.\beta}\right)\\
     &+ \frac{1}{\alpha^2\min\{\eps^2,1\}} \cdot \frac{2 k\left(\log k\right)^{\log\log\frac{3}{\beta}}}{(\log k)^{\log\frac{3}{2}(5+4\log\log \frac{3}{\beta})(1+1/\log k)}} \log\log k \cdot \log(2\log \frac{3}{\beta}) \\
     & = O\left(\frac{k\log \frac{1}{\beta}}{\alpha^2\min\{\eps^2,1\}}\right).
    \end{aligned}
\end{equation}

Finally we find the number of samples needed to run \textsc{mde-variant}. Let $\tilde{\cF} = \cR_1 \cup \cR_2 \cup \cK_1$. We know that MDE-variant makes $\frac{|\cF|(|\cF|-1)}{2}$ queries and all of them are critical. Therefore, we can write
\begin{equation}\label{eq3:thm_local_alg}
    \begin{aligned}
       m_3 &\leq  \frac{1}{\alpha^2\min\{\eps^2,1\}} \left(\frac{k(\log\frac{1}{\beta})^{t'}}{2^{t\log\frac{3}{2}}k'^{\frac{2^{t'}-1}{2^{(t'+1)}}}} + 8\log\frac{3}{\beta}2^{t\cdot\log\frac{3}{2}}+2k'^{\frac{2^{t'}}{2^{(t'+1)}}}\right)^2 \cdot\\
      &\log \left(\frac{k(\log\frac{1}{\beta})^{t'}}{2^{t \cdot \log\frac{3}{2}}k'^{\frac{2^{t'}-1}{2^{(t'+1)}}}} + 8\log\frac{3}{\beta}2^{t\cdot\log\frac{3}{2}}+2k'^{\frac{2^{t'}}{2^{(t'+1)}}}\right)\\
      & =O\left( \frac{k(\log \frac{1}{\beta})^2}{\alpha^2\min\{\eps^2,1\}} \right).
    \end{aligned}
\end{equation}
Combining Equations~\ref{eq1:thm_local_alg}, \ref{eq2:thm_local_alg}, and \ref{eq3:thm_local_alg} concludes that $m_1+m_2+m_3 = O\left( \frac{k(\log \frac{1}{\beta})^2}{\alpha^2\min\{\eps^2,1\}} \right)$ as desired.
\end{proof}

\subsection{Proof of Lemma~\ref{lemma:sqo_simulation}}
\begin{proof}
    We first prove that for any $i\in [n]$ with probability at least $\beta'=\beta/m$ we have that the output of oracle for the query $W_i$ is close to $\expects{x\sim h}{\indicator{x\in W_i}}$. For any $p(i-1)+1\leq j\leq p.i$, let $y_j = R_{\eps}(\indicator{{x_j}\in W_i})$  and $z_j$ be the random variable defined below
    \begin{equation*}
         z_j = \frac{e^\eps +1}{e^\eps - 1}\left( R_\eps\left(\indicator{{x_j}\in W_i}\right) - \frac{1}{e^\eps+1}\right) = \frac{e^\eps +1}{e^\eps - 1}\left( y_j - \frac{1}{e^\eps+1}\right).
    \end{equation*}
    We now want to apply the Hoeffding's inequality on random variables $z_j$. We know that $\expect{y_j} = \frac{e^\eps}{e^\eps+1}\indicator{x_j\in W_i} + \frac{1}{e^\eps + 1}(1-\indicator{x_j\in W_i}) = \frac{e^\eps-1}{e^\eps+1}\indicator{x_j\in W_i}+\frac{1}{e^\eps+1}$ where the expectation is taken over the randomness of $R_\eps$. We can then compute the expected value of $z_j$ as follows.
    \begin{equation*}
    \begin{aligned}
    &\expects{x_j,R_\eps}{z_j} = \frac{e^\eps +1}{e^\eps - 1}\left( \expect{y_j} - \frac{1}{e^\eps+1}\right) = \frac{e^\eps +1}{e^\eps - 1}\left( \expect{\frac{e^\eps-1}{e^\eps+1}\indicator{x_j\in W_i}+\frac{1}{e^\eps+1}} - \frac{1}{e^\eps+1}\right)\\
    & = \frac{e^\eps +1}{e^\eps - 1}\left( \frac{e^\eps-1}{e^\eps+1}\expect{\indicator{x_j\in W_i}}+\frac{1}{e^\eps+1} - \frac{1}{e^\eps+1}\right) = \expect{\indicator{x_j\in W_i}}.
    \end{aligned}
    \end{equation*}
    We can therefore apply the Hoeffding's inequality (Lemma~\ref{lemma:hoeffding}) on $z_j$ and write that
    \begin{equation*}
    \begin{aligned}
        & \prob{\left|\frac{e^\eps +1}{e^\eps - 1}\left(\frac{1}{p}\sum_{j=p(i-1)+1}^{p.i} R_\eps\left(\indicator{{x_j}\in W_i}\right) - \frac{1}{e^\eps+1}\right) - \expect{\indicator{x_j\in W_i}}\right| > \alpha} \leq \exp\left(-\frac{2p\alpha^2}{(\frac{e^\eps+1}{e^\eps-1})^2}\right). 
    \end{aligned}
    \end{equation*}
    If $\eps \leq 1$ then $\left(\frac{e^\eps+1}{e^\eps -1}\right)^2 = O\left((\frac{1}{\eps})^2\right)$.  If $\eps > 1$ then for any constant $c>3$ we know that $\ln{\frac{c+1}{c-1}} < 1< \eps$ and, thus, $\left(\frac{e^\eps+1}{e^\eps -1}\right)^2 < c^2 = O(1)$. Therefore, we can conclude that $\left(\frac{e^\eps+1}{e^\eps -1}\right)^2 = O\left(\frac{1}{\min\{\eps^2,1\} }\right)$. Setting $p = \frac{\log 1/\beta'}{\min \alpha^2\{\eps^2,1\}} = \frac{\log m/\beta}{\alpha^2\min \{\eps^2,1\}}$ implies that
    \begin{equation*}
        \prob{\left|\frac{e^\eps +1}{e^\eps - 1}\left(\frac{1}{p}\sum_{j=p(i-1)+1}^{p.i} R_\eps\left(\indicator{{x_j}\in W_i}\right) - \frac{1}{e^\eps+1}\right) - \expect{\indicator{x_j\in W_i}}\right| > \alpha} \leq \beta' = \frac{\beta}{m}.
    \end{equation*}
    A simple union bound argument concludes concludes that
    \begin{equation*}
        \forall U\subset [k], |U|=m,\,\prob{\sup_{i\in U}\left| \cO_{h}^{RR}\left(W,\alpha,\beta,pk \right)_i-\expects{x\sim h}{\indicator{x\in W_i}}\right|\geq \alpha} \leq \beta,
    \end{equation*}
    which proves that $\cO_{h}^{RR}$ is a valid SQO for $h$.
\end{proof}

\section{Pseudocode of Classical Algorithms}\label{app:algo}
\begin{algorithm}
  \SetKwFunction{Scheffe}{SCHEFF\'E}
  \SetKwProg{Fn}{procedure}{:}{end procedure}
  \SetFuncSty{textsc}
  \KwIn{A pair of distributions $f_1$ and $f_2$ on $\cX$, $y\in\bR$}
  \KwOut{$f_1$ or $f_2$}
  \Fn{\Scheffe{$f_1,f_2,y$}}{
    $B_s \gets Sch(f_1,f_2)$ \;\\
    \uIf{ $|f_1[B_s] - y| \leq |f_1[B_s] - y|$}{
     \Return $f_1$\;
    }
    \Else{
     \Return $f_2$\;
    }
    
    }
  \caption{Scheff\'e test}\label{alg:scheffe}
\end{algorithm}
\begin{algorithm}
   \SetFuncSty{textsc}
  \SetKwFunction{round}{Round-Robin}
  \SetKwProg{Fn}{procedure}{:}{end procedure}
  \KwIn{A set $\cF=\{f_1,\ldots,f_k\}$ of $k$ distributions, Oracle $\cO_h$, parameters $\alpha,\beta>0$}
  \KwOut{A distribution $\hat{f}\in \cF$ such that $d_{TV}(h,\hat{f}) \leq 9d_{TV}(h,\cF)+\alpha$ with probability at least $1-\beta$.}

 \Fn{\round{$\cF,\cO_h,\alpha,\beta$}}{
$\forall f\in \cF$, $w[f] \gets 0$ \Comment{To record the number of wins of $f$}\;\\
 \For{all $i,j\in[k]$ where $j>i$}{
     $B_{ij} \gets Sch(f_i,f_j)$\;
     }
     $W = \left(B_{ij}\right)_{i=1,j>i}^{k},$\;$\left(y_{ij}\right)_{i=0, j>i}^{k} \gets \cO_h(W,\alpha,\beta)$\;
 
  \For{all $i,j\in[k]$ where $j>i$}{

 \uIf{$\FuncSty{SCHEFF\'E}(f_i,f_j,y_{ij})$ returns $f_i$}{
 $w[f_i] \gets w[f_i]+1$\; }
    \Else{ $w[f_j] \gets w[f_j]+1$\;}
  
  }
  \Return $\hat{f}$ with maximum number of wins $w[\hat{f}]$\;
 } 
 \caption{Round-Robin}\label{alg:rr}
\end{algorithm}
\begin{algorithm}
  \SetFuncSty{textsc}
  \SetKwFunction{mde}{MDE-Variant}
  \SetKwProg{Fn}{procedure}{:}{end procedure}
  \KwIn{A set $\cF=\{f_1,\ldots,f_k\}$ of $k$ distributions, Oracle $\cO_h$, parameters $\alpha,\beta>0$}
  \KwOut{A distribution $\hat{f}\in\cF$ such that $d_{TV}(h,\hat{f}) \leq 3d_{TV}(h,\cF)+\alpha$ with probability at least $1-\beta$.}

 \Fn{\mde{$\cF,\cO_h,\alpha,\beta$}}{
$\forall f\in\cF, w[f] \gets 0$\;\\
 \For{every $f\in \cF$}{
    \For{every $f'\in\cF\setminus f$}{
    $B_i \gets Sch(f,f')$, $i\gets i+1$\;
 }
 $W = (B_i)_{i=1}^{k-1},\, (y_i)_{i=1}^{k-1}=\cO_h(W,\alpha,\beta)$ \;\\
     $w[f] \gets \sup_{B_i} \left|f[B_i] - y_i\right|$ \;
 }
 \Return $\hat{f} = \arg\min_{f\in\cF} w[f]$\;
 }
\caption{MDE-Variant}\label{alg:mde}
\end{algorithm}

\end{document}